\documentclass[letterpaper, 10 pt, conference]{ieeeconf}  % Comment this line out if you need a4paper

\IEEEoverridecommandlockouts                    % This command is only needed if 
\usepackage[english]{babel}
\usepackage[utf8]{inputenc}
\usepackage{amsmath}
\usepackage{float}
\usepackage{graphicx,subfigure}
\usepackage[utf8]{inputenc}
\usepackage[english]{babel}
\usepackage{amsfonts}
\usepackage{comment} 
\usepackage{tikz}
\usepackage{mathtools}
\usepackage{fixltx2e}
\usepackage{pifont}
\usepackage{caption}
\usepackage[colorinlistoftodos]{todonotes}

\usepackage{amsmath,amsfonts,amssymb,amsthm,epsfig,epstopdf,array}

\newtheorem{theorem}{Theorem}

\newtheorem{assumption}{Assumption}
\newtheorem{mydef}{Definition}
\usepackage{caption}

\newcommand{\norm}[1]{\left\lVert#1\right\rVert}
\usepackage{verbatim}
\usepackage{balance}
\let\labelindent\relax
\usepackage{enumitem}
\usepackage[bottom]{footmisc}
\usepackage[utf8]{inputenc}
\usepackage[english]{babel}

\usepackage[linkcolor=blue,colorlinks=true]{hyperref}
 \usepackage{cleveref}
 \usepackage{slashbox,booktabs,amsmath}

\title{\LARGE \bf
Noncooperative Herding With Control Barrier Functions: \\Theory and Experiments
}

\author{Jaskaran Grover$^{*1}$, Nishant Mohanty$^{*1}$, Wenhao Luo$^{2}$, Changliu Liu$^{1}$, Katia Sycara$^{1}$% <-this % stops a space
\thanks{$^{1}$J. Grover, N. Mohanty, C. Liu, K. Sycara are with the Robotics Institute at
		Carnegie Mellon University, 5000 Forbes Avenue, Pittsburgh, PA 15213, USA.
		{\tt\small \{nishantm,jaskarag,cliu6,sycara\}@andrew.cmu.edu}}
		\thanks{$^{2}$W. Luo is with the Department of Computer Science, University of North Carolina at Charlotte.
		{\tt\small wenhao.luo@uncc.edu}}
		\thanks{*Denotes Equal Contribution}}
		
\begin{document}

\maketitle
%%%%%%%%%%%%%%%%%%%%%%%%%%%%%%%%%%%%%%%%%%%%%%%%%%%%%%%%%%%%%%%%%%%%%%%%%%%%%%%%
\begin{abstract}
In this paper, we consider the problem of protecting a high-value unit from inadvertent attack by a group of agents using defending robots. Specifically, we develop a control strategy for the defending agents that we call ``dog robots" to prevent a flock of ``sheep agents" from breaching a protected zone. We take recourse to control barrier functions to pose this problem and exploit the interaction dynamics between the sheep and dogs to find dogs' velocities that result in the sheep getting repelled from the zone. We solve a QP reactively that incorporates the defending constraints to compute the desired velocities for all dogs. Owing to this, our proposed framework is composable \textit{i.e.} it allows for simultaneous inclusion of multiple protected zones in the constraints on dog robots' velocities.  We provide a theoretical proof of feasibility of our strategy for the one dog/one sheep case. Additionally, we provide empirical results of two dogs defending the protected zone from upto ten sheep averaged over a hundred simulations and report high success rates. We also demonstrate this algorithm experimentally on non-holonomic robots. Videos of these results are available at \url{https://tinyurl.com/4dj2kjwx}.

% In this paper, we look at the specific challenge of
% protecting an asset against an adversarial swarm. Autonomous defensive agents are tasked with protected a
% High Value Unit (HVU) from an incoming swarm attack.

% The proposed method is a generalised approach, for which we chose to demonstrate its application using this herding problem. Here the sheep robots follow a flocking dynamics to collectively reach a common goal position. Now given the dynamics of the sheep robots we develop an optimization based controller, the constraints of which are derived from the dynamics. It provides us provable guarantees for the completion of task at hand, which is in this case the preventing the breaching of a given protected zone. Also as it is a constraint based method it provides us flexibility of composing multiple tasks at hand as long as there exists a feasible solution. The performance of proposed method have been shown using numerical simulations and hardware experiments on robots.

\end{abstract}
%%%%%%%%%%%%%%%%%%%%%%%%%%%%%%%%%%%%%%%%%%%%%%%%%%%%%%%%%%%%%%%%%%%%%%%%%%%%%%%%
\section{Introduction}
In the last decade, multi-robot systems (MRSs) have advanced from being researched in labs to being deployed in the real-world for solving practical problems \cite{d2012guest,d2003distributed}, \cite{kazmi2011adaptive}. The redundancy offered by an aggregated system provides resilience to faults and distributed acquisition of information. Several control algorithms have been developed that make multiple robots come together to solve team-level, global tasks using local interaction rules \cite{ji2007distributed,lin2004multi}. These algorithms are (a) local (\textit{i.e.} individual robots act on information locally available to them), (b) safe (\textit{i.e.} result in collision-free motions amongst robots) and (c) emergent (\textit{i.e.} global properties result from using local interaction rules) \cite{reynolds1987flocks}.  

These characteristics can be treated as the insider's perspective \textit{i.e.} principles borne in mind by the control engineer when programming their \textit{own} robots for a given task. Complementary to this is the outsider's perspective \textit{i.e.} the perspective of an external agent watching a group carry out a task by executing motions consistent with these characteristics \cite{gong2020partial}. Viewing the the motion of a group from the vantage point of an external observer is equally important. For example, if group is adversarial, potentially by posing a threat to a high-value unit, then the observer must predict the group's motion and conscript robots (the defenders) to defend the unit \cite{walton2021defense,tsatsanifos2021modeling}. This requires the observer to orchestrate motions for their robots to prevent breach of the high-value unit.  In this paper, we investigate how to develop provably correct control inputs for a group of defenders (``dog robots") to prevent another group (the ``sheep agents") from breaching a protected zone.  This is a challenging problem because the dog robots cannot directly command the actuators of the sheep agents, they must rely on their interaction dynamics (collision-avoidance behavior) with the sheep agents to influence the sheep's behavior. This results in a non-collocated control problem. Additionally, this is also challenging because usually there are not as many defending robots as agents in the herd. Therefore, from the perspective of the dog robots, the control problem can become highly underactuated. 

\begin{figure}
	\centering     %%% not \center
	\subfigure[Preventing breaching of protected zone]{\label{fig:dA6}\includegraphics[trim={1.1cm 1.5cm 1.5cm 1.5cm},clip,width=0.45\columnwidth]{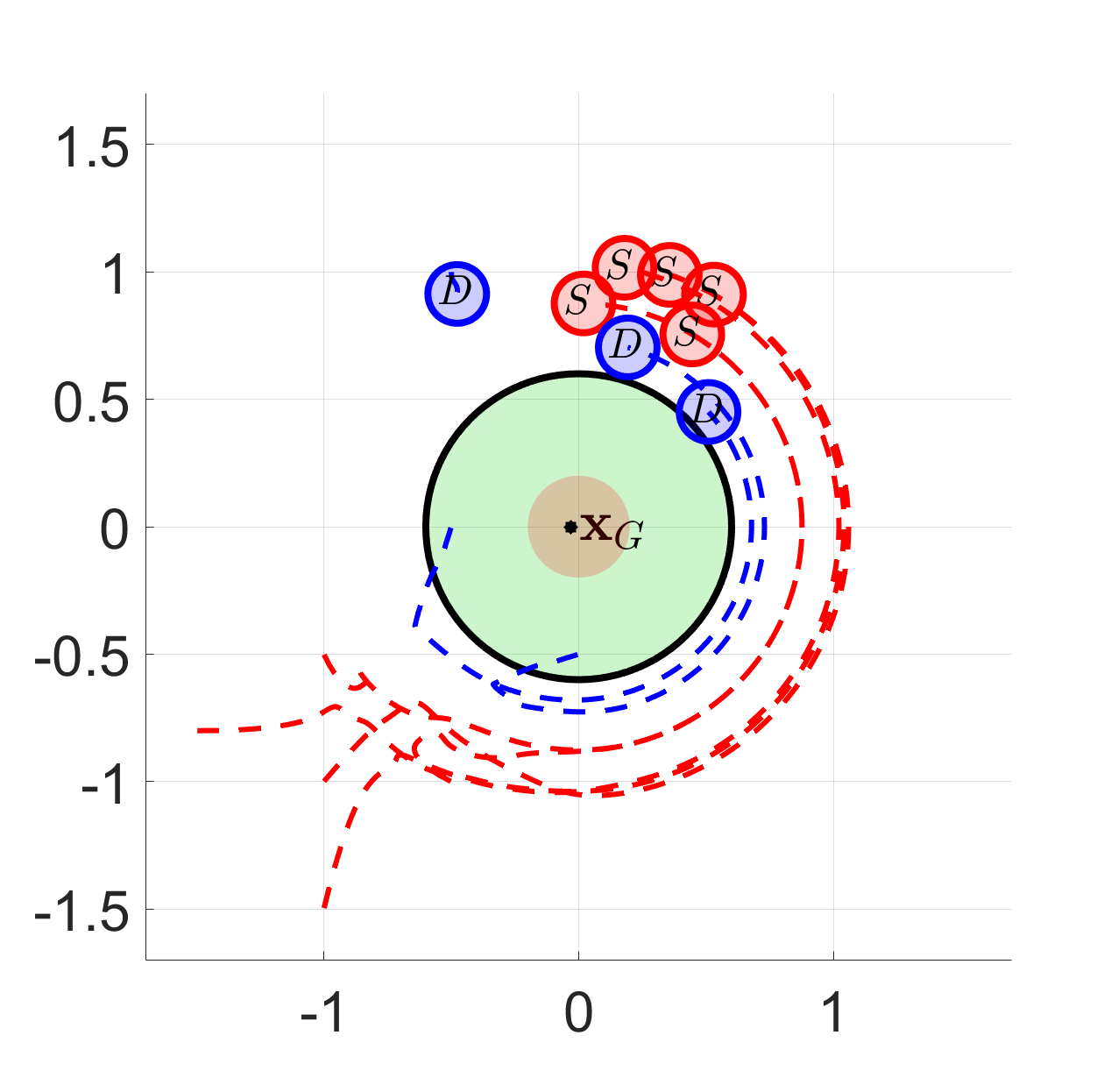}}
	\subfigure[Preventing escape from protected zone]{\label{fig:dB6}\includegraphics[trim={1.1cm 1.5cm 1.5cm 1.5cm},clip,width=0.45\columnwidth]{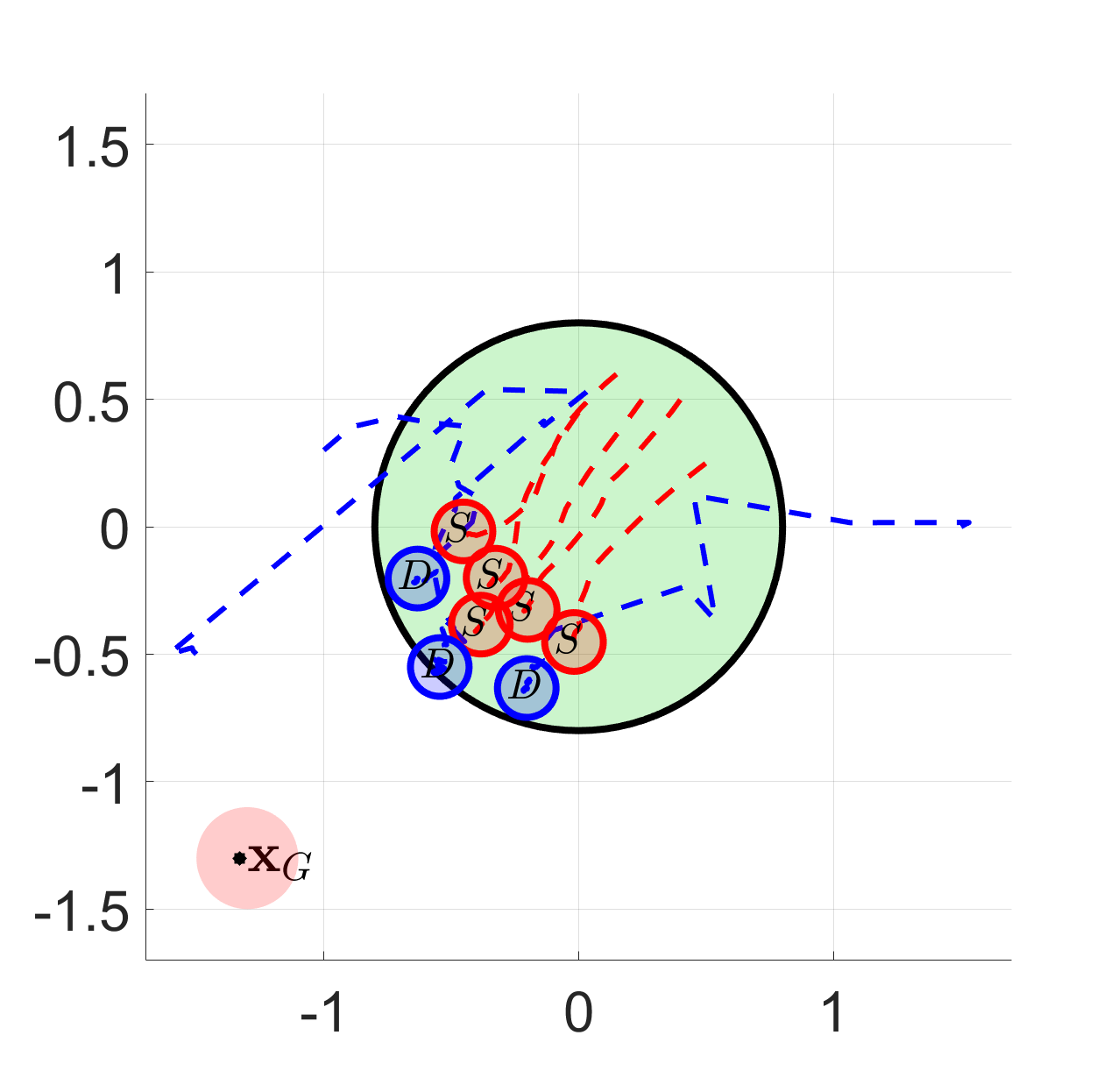}}
	\caption{Demonstration of our results showing (a) how to prevent sheep (red) from breaching a protect zone (green) and (b) preventing sheep (red) from escaping the protected zone using dog robots (blue).  }
	\label{fig:escape_and_defend}
\end{figure}

In this paper, we investigate how to solve this problem using ideas from control barrier functions. Specifically, we develop a centralized control technique that computes velocity inputs for the dog robots to ensure that the sheep agents do not breach a protected zone. We convert these requirements to constraints on the velocities of the dog robots. Our proposed framework is compositional in nature \textit{i.e.} we can consider more sheep as well as more protected zones by just adding more constraints on the velocities. Furthermore, our approach relies on using automatic differentiation and symbolic computation tools, owing to which, we can easily change behavioral requirements from the sheep. For example, instead of preventing them from breaching a protected zone  (Fig. \ref{fig:dA6}), we can prevent them from escaping a zone  (Fig. \ref{fig:dB6}).  We provide numerical results showing the success of our approach for multiple dogs v/s multiple sheep agents. Additionally, to test the repeatability of our algorithm, we conduct Monte Carlo simulations with increasing number of dogs and sheep averaged over 100 runs each and show high-success rates. Finally, we demonstrate our algorithm on real robots and demonstrate that we can prevent breaching of multiple zones from two sheep using one dog robot.

The outline of this paper is as follows: in section \ref{priorwork} we briefly review the prior work in this area. In section \ref{problemformulation}, we give a mathematical formulation of the problem statement. In section \ref{controllerformulation}, we show how to use control barrier functions to derive constraints on velocities of dog robots to pose the requirement of defense against the sheep. We consider additional collision avoidance constraints on dog robots' velocities. In section \ref{results}, we provide both simulation as well as experimental results demonstrating our approach. Finally, we summarize in section \ref{conclusions} with directions for future work.
\section{Prior Work}
 \label{priorwork}
Influencing group behavior has applications beyond just the adversarial context. For example, shepherding behaviors, specifically, are one class of flocking behaviors in which one or more external agents (called shepherds) attempt to control the motion of another group of agents (called a flock) by exerting repulsive forces on them \cite{lien2004shepherding,pierson2017controlling}.  A successful practical demonstration of robotic herding was achieved in the Robot Sheepdog Project \cite{vaughan1998robotA,vaughan2000experiments}. Here an autonomous wheeled mobile robot (the external agent/shepherd) was used to gather a flock of ducks and manoeuvred them to a specified goal position. 

Several prior works have considered the problem of noncooperative shepherding using robots. Some of these include \cite{pierson2017controlling},\cite{pierson2015bio},\cite{licitra2017singleA},\cite{licitra2017singleB},\cite{sebastian2021multi},\cite{bacon2012swarm}. They refer to the shepherding problem as noncooperative because the flock agents are not necessarily adversarial \textit{i.e.} they do not work against the robots, but at the same time are not cooperative because the flock agents repel from the robots. These works exploit this repulsive interaction to develop feedback controllers for the robots to steer the flock agents to a designated region. While successful, one issue common among these works is that they fail to consider the self-motivated dynamics of the flock agents \textit{i.e.} their nominal dynamics without any robots in the picture. As a result, the flock agents' motions are solely driven by repulsions from the robots.  Additionally, these approaches tend to be handcrafted for generating a specific behavior in the sheep, for example: herding to a given location. Finally, many papers do not consider scalability with respect to the number of agents.

Differently from prior work, we do not omit the self-motivated dynamics in the sheeps' motions. We synthesize the inputs for the dogs while considering cohesion, inter-sheep and dog/sheep repulsions and the sheep agents' attraction to their goal (the self-motivated part). Moreover, our proposed approach uses control barrier functions which only requires expected behaviors from sheep to be expressed as symbolic functions. Using automatic differentiation, we can generate constraints on dog velocities for any given behavioral requirement from the sheep. Lastly, in our Monte Carlo study, we obtain high success rates even when there are many more sheep than the number of dog robots in the system. This provides an empirical evidence of scalability of our approach.

%%%%%%%%%%%%%%%%%%%%%%%%%%%%%%%%%%%%%%%%%%%%%%%%%%%%%%%%%%%%%%%%%%%%%%%%%%%%%%%%
\section{Problem Formulation}
 \label{problemformulation}
Suppose there are $n$ ``sheep" agents (the herd) and $m$ dog robots (the defenders). We assume that the sheep are exhibiting flocking dynamics \textit{i.e.} moving towards a common goal while staying close enough to each other and repelling from the dogs. Given this dynamics, it is possible that while en-route to their goal, they end up breaching a high-value unit \textit{i.e.} the protected zone. From the perspective of the dogs, the sheep represent a non-cooperative group because they are not intentionally aiming towards the protected zone but may inadvertently end up breaching it. Therefore, the objective of the dog robots is to steer the sheep away from the protected zone. Let us pose this requirement mathematically.

Denote the position of the $i^{th}$ sheep as $\boldsymbol{x}_{S_i} \in \mathbb{R}^2$ and the collective positions of the herd as $\boldsymbol{x}^{all}_{S} \coloneqq (\boldsymbol{x}_{S_1},\boldsymbol{x}_{S_2},...,\boldsymbol{x}_{S_n})$. Likewise, we denote the position of the $k^{th}$ dog as $\boldsymbol{x}_{D_k} \in \mathbb{R}^2$ and the collective positions of the defending robots' group as $\boldsymbol{x}^{all}_{D} \coloneqq (\boldsymbol{x}_{D_1},\boldsymbol{x}_{D_2},...,\boldsymbol{x}_{D_m})$.   We assume both sheep and dogs have single-integrator dynamics \textit{i.e.} they are velocity controlled. For the $i^{th}$ sheep, we have:
\begin{align}
\label{sheepdynamics}
    \Dot{\boldsymbol{x}}_{S_i} &= \boldsymbol{u}_{S_i}   \\
    &= k_{S} \sum_{j \in \mathcal{S}}\left(1-\frac{R_{S}^{3}}{\norm{\boldsymbol{x}_{S_j} - \boldsymbol{x}_{S_i}}^{3}}\right) (\boldsymbol{x}_{S_j} - \boldsymbol{x}_{S_i}) \nonumber \\
&+k_{G}\left(\boldsymbol{x}_{G}-\boldsymbol{x}_{S_i}\right)
+k_{D} \sum_{k \in \mathcal{D}}  \frac{\boldsymbol{x}_{S_i} - \boldsymbol{x}_{D_k}}{\norm{\boldsymbol{x}_{S_i} - \boldsymbol{x}_{D_k}}^{3}} \nonumber \\
&\coloneqq \boldsymbol{f}_i(\boldsymbol{x}_{S_1},...,\boldsymbol{x}_{S_n},\boldsymbol{x}_{D_1},...,\boldsymbol{x}_{D_n})
\end{align}
Here the first term represents cohesion of the flock, the second represents attraction to goal and the third represents repulsion from dog robots. The attraction to the goal represents the self-motivated part of the dynamics of the sheep agents. This term is often neglected in prior work. $R_S$ is the safety radius for sheep $i$ to avoid collisions with the other sheep, $x_G$ is its desired goal position (common for all sheep) and $k_S,k_G,k_D$ are proportional gains corresponding to forces in the dynamics. For each dog we have:
\begin{align}
\label{doginput}
 \Dot{\boldsymbol{x}}_{D_k} = \boldsymbol{u}_{D_k} \hspace{0.2cm} \forall k \in \{1,2,\cdots,m\}
\end{align}
% Denoting $\boldsymbol{x}^{all}_D=(\boldsymbol{x}_{D_1},\boldsymbol{x}_{D_2},\cdots,\boldsymbol{x}_{D_m})$, we have
% \begin{align}
% \label{dogsall}
%  \Dot{\boldsymbol{x}}^{all}_{D} = \boldsymbol{u}^{all}_{D} 
% \end{align}
We denote the protected zone as $\mathcal{P} \subset \mathbb{R}^2$ and for this paper, assume that it is a disc centered at $\boldsymbol{x}_{p}$ and radius $R_p$:
\begin{align}
\label{protected_zone_def}
    \mathcal{P} \coloneqq \{\boldsymbol{x} \in \mathbb{R}^2 \vert \norm{\boldsymbol{x}-\boldsymbol{x}_p}\leq R_p\}
\end{align}
We denote the set excluding the protected zone as $\mathcal{P}^c \coloneqq \mathbb{R}^2\backslash \mathcal{P}$. The sheep are assumed to have no knowledge about the presence of $\mathcal{P}$. The dog robots need to ensure that the sheep remain in $\mathcal{P}^c$ if they are initially in $\mathcal{P}^c$ by finding suitable control inputs $\{\boldsymbol{u}_{D_1},\cdots,\boldsymbol{u}_{D_m}\}$. We make the following assumption on the dog's knowledge before posing the problem:
\begin{assumption}
	\label{ass1}
	The dog robots have knowledge about the sheep's dynamics \textit{i.e.} \eqref{sheepdynamics} and can measure the sheep's positions accurately. 
\end{assumption}
This is not a stringent assumption because if the  dynamics are unknown, the dog robots can learn the dynamics online using multiagent system identification algorithms, some of which we have developed in our prior work \cite{grover2020parameter,grover2020feasible} and use certainty equivalence to design the controllers. We can pose the dog robots' problem as follows:
\begin{mydef}
Assuming that the initial positions of the sheep $\boldsymbol{x}^{all}_S(0) \in  \mathcal{P}^c$, the dog robots' problem is to synthesize controls $\{\boldsymbol{u}_{D_1},\cdots, \boldsymbol{u}_{D_m}\}$ such that $\boldsymbol{x}^{all}_S(t) \in  \mathcal{P}^c$ $\forall t\geq 0$. If $\boldsymbol{x}^{all}_S(0) \notin  \mathcal{P}^c$, the dog robots' problem is to synthesize controls $\{\boldsymbol{u}_{D_1},\cdots, \boldsymbol{u}_{D_m}\}$ such that $\boldsymbol{x}^{all}_S(t) \leadsto
 \mathcal{P}^c$  in a finite time.
\end{mydef}
Additionally, we require that the dog robots never collide with the sheep.
% \begin{itemize}
%   \item not considering collision avoidance with the sheep robot.
%   \item having an addition constraint to avoid collision with the sheep robot. 
% \end{itemize}
In the next section, we show how to address this problem using control barrier functions.
\section{Controller Design}
 \label{controllerformulation}

In this section, we discuss our proposed approach to solve the problem of defending the protected zone as stated before. Given the protected zone as defined \eqref{protected_zone_def}, we first pose the requirement for defending against one sheep, say sheep $i$ located at $\boldsymbol{x}_{S_i}$. Subsequently, we will generalize this to the rest of the sheep in the herd. For this sheep, define a safety index $h(\cdot):\mathbb{R}^2 \longrightarrow \mathbb{R}$ as follows:
\begin{align}
    h = \norm{\boldsymbol{x}_{S_i} - \boldsymbol{x}_{p}}^2 - R_p^2
\end{align}
By construction, $h \geq 0$ $\forall \boldsymbol{x}_{S_i} \in \mathcal{P}^c$ \textit{i.e.} non-negative whenever $i$ is on the boundary or outside the protected zone. Thus, assuming that at $t=0$, $h (\boldsymbol{x}_{S_i}(0)) \geq 0$, we require $h (\boldsymbol{x}_{S_i}(t)) \geq 0$ $\forall t \geq 0$. Treating $h(\cdot)$ as a control barrier function \cite{ames2019control}, this can be achieved if the derivative of $h(\cdot)$ satisfies the following constraint:
\begin{align}
\label{hdot}
    &\Dot{h}(\boldsymbol{x}_{S_1},\cdots,\boldsymbol{x}_{S_n},\boldsymbol{x}_{D_1},\cdots,\boldsymbol{x}_{D_m}) + p_1h(\boldsymbol{x}_{S_i}) \geq 0  \nonumber \\
    \implies &2(\boldsymbol{x}_{S_i} - \boldsymbol{x}_{p})^T\Dot{\boldsymbol{x}}_{S_i}  + p_1h(\boldsymbol{x}_{S_i}) \geq 0 \nonumber \\
    \implies & 2(\boldsymbol{x}_{S_i} - \boldsymbol{x}_{p})^T \boldsymbol{f}_{i} + p_1h(\boldsymbol{x}_{S_i}) \geq 0
\end{align}
Define $\boldsymbol{x}=(\boldsymbol{x}^{all}_S,\boldsymbol{x}^{all}_D)$, we rewrite this as
\begin{align}
\label{hdot2}
 2(\boldsymbol{x}_{S_i} - \boldsymbol{x}_{p})^T \boldsymbol{f}_{i}(\boldsymbol{x}) + p_1h(\boldsymbol{x}_{S_i}) \geq 0
\end{align}
Here $p_1$ is a design parameter that we choose to ensure that
\begin{align}
    \label{eq:condition_on_p1}
    p_1 > 0 \quad \text{and} \quad p_1 > -\frac{\Dot{h}(\boldsymbol{x}(0))}{h(\boldsymbol{x}(0))}
\end{align}
The first condition on $p_1$ requires that the pole is real and negative. The second depends on the initial positions $\boldsymbol{x}(0)$ of all the sheep and dogs relative to the protected zone.  Now while \eqref{hdot} depends on the positions of the sheep and dogs,  it is the velocities of the dogs that are directly controllable not their positions \eqref{doginput}. Since $\boldsymbol{u}^{all}_D$ does not show up in \eqref{hdot}, we define another function $v(\cdot):\mathbb{R}^{2(m+n)} \longrightarrow \mathbb{R}$:
\begin{align}
\label{vdef}
    v = \Dot{h} + p_1h
\end{align}
Like before, in order to ensure $v \geq 0$ is always maintained, its derivative needs to satisfy 
\begin{align}
\label{vdot}
    \Dot{v}(\boldsymbol{x}) + p_2v(\boldsymbol{x}) \geq 0.
\end{align}
Here $p_2$ is another design parameter which we choose $p_2$ to ensure that the following is satisfied at $t=0$
\begin{align}
\label{eq:condition_on_p2}
    p_2 > 0 \quad \text{and} \quad p_2 > -\frac{\Ddot{h}(\boldsymbol{x}(0)) + p_1\Dot{h}(\boldsymbol{x}(0))}{\Dot{h}(\boldsymbol{x}(0)) + p_1h(\boldsymbol{x}(0))}
\end{align}
Using \eqref{vdef} in  \eqref{vdot}, we get:
\begin{align}
\label{timederivatives}
  \Ddot{h}(\boldsymbol{x}) + (p_1+p_2)\Dot{h}(\boldsymbol{x}) + p_1p_2h(\boldsymbol{x}) \geq 0 \nonumber \\
  \implies \Ddot{h}(\boldsymbol{x}) + \alpha \Dot{h}(\boldsymbol{x}) + \beta h(\boldsymbol{x}) \geq 0 
\end{align}
where we have defined $\alpha \coloneqq p_1 +p_2$ and $\beta \coloneqq p_1p_2$.
The time derivatives of the control-barrier function $h(\cdot)$ required  in \eqref{timederivatives} are obtained as:
\begin{align}
\label{hd}
    \Dot{h}(\boldsymbol{x}) &= 2(\boldsymbol{x}_{S_i} - \boldsymbol{x}_{P})^T\Dot{\boldsymbol{x}}_{S_i} \nonumber \\
    &=2(\boldsymbol{x}_{S_i} - \boldsymbol{x}_{P})^T \boldsymbol{f}_i(\boldsymbol{x}_{S_1},\cdots,\boldsymbol{x}_{S_n},\boldsymbol{x}_{D_1},\cdots,\boldsymbol{x}_{D_m}) 
\end{align}
    \begin{align}
    \label{hdd}
    \Ddot{h}(\boldsymbol{x}) &= 2\Dot{\boldsymbol{x}}_{S_i}^T\Dot{\boldsymbol{x}}_{S_i} \nonumber \\&+ 2(\boldsymbol{x}_{S_i}-\boldsymbol{x}_{P})^T\bigg(\sum_{j=1}^{n}\mathbb{J}_{ji}^S \Dot{\boldsymbol{x}}_{S_i} + \sum_{k=1}^m\mathbb{J}_{ki}^D\boldsymbol{u}_{D_k}\bigg)  \nonumber \\
     &= 2 \boldsymbol{f}^T_{i} \boldsymbol{f}_{i}  \nonumber  \\&+ 2(\boldsymbol{x}_{S_i}-\boldsymbol{x}_{P})^T\bigg(\sum_{j=1}^{n}\mathbb{J}_{ji}^S \boldsymbol{f}_{i} + \sum_{k=1}^m\mathbb{J}_{ki}^D\boldsymbol{u}_{D_k}\bigg)
\end{align}
where $\mathbb{J}^S_{ji}$  and $\mathbb{J}^D_{ki}$ are
\begin{align}
    \mathbb{J}^S_{ji} &\coloneqq \nabla_{\boldsymbol{x}_{S_j}} \boldsymbol{f}_{i} (\boldsymbol{x}_{S_1},\cdots,\boldsymbol{x}_{S_n},\boldsymbol{x}_{D_1},\cdots,\boldsymbol{x}_{D_m})  \nonumber \\
     \mathbb{J}^D_{ki} &\coloneqq \nabla_{\boldsymbol{x}_{D_k}} \boldsymbol{f}_{i} (\boldsymbol{x}_{S_1},\cdots,\boldsymbol{x}_{S_n},\boldsymbol{x}_{D_1},\cdots,\boldsymbol{x}_{D_m})  \nonumber
\end{align}
Note here that $\Ddot{h}(\boldsymbol{x})$ contains the velocities of dogs as we wanted. Using  \eqref{hd} and \eqref{hdd} in \eqref{timederivatives}, we get the following linear constraints on dog velocities to ensure that the $i^{th}$ sheep stays outside the protected zone $\mathcal{P}$: 
\begin{align}
    \label{herdingcon1}
    A^H_i \boldsymbol{u}^{all}_D\leq b^H_i, \hspace{0.5cm} \mbox{where}
\end{align}
\begin{align*}
    A^H_{i} &\coloneqq  (\boldsymbol{x}_{P} -\boldsymbol{x}_{S_i})^T
    \begin{bmatrix}
     \mathbb{J}_{1i}^D & \mathbb{J}_{2i}^D & ..... & \mathbb{J}_{mi}^D
    \end{bmatrix}\\
      b^H_{i} &\coloneqq\boldsymbol{f}^T_i\boldsymbol{f}_i
      + (\boldsymbol{x}_{S_i} - \boldsymbol{x}_{P})^T\sum_{j=1}^{n}\mathbb{J}_{ji}^S \boldsymbol{f}_j\\
      &+ \alpha(\boldsymbol{x}_{S_i}-\boldsymbol{x}_{P})^T\boldsymbol{f}_i + \beta\frac{h}{2} 
\end{align*}
To ensure all $n$ sheep stay away from $\mathcal{P}$, we compose constraints \eqref{herdingcon1} for all the herd as follows:
\begin{align}
\label{allherd}
    \left[\begin{matrix}
A^H_1 \\
\vdots \\
A^H_n
\end{matrix}\right] \boldsymbol{u}^{all}_D \leq \left[\begin{matrix}
b^H_1 \\
\vdots \\
b^H_n
\end{matrix}\right] \implies \mathcal{A}^H\boldsymbol{u}^{all}_D \leq \boldsymbol{b}^H
\end{align}
Here $\mathcal{A}^H \in \mathbb{R}^{n \times 2m}$  and $\boldsymbol{b}^H \in \mathbb{R}^{n}$. Given these constraints on the dogs' velocities, we can pose the following QP that searches for the min-norm velocities that satisfies these constraints
\begin{align}
\label{dog_control}
    \boldsymbol{u}^{*all}_D = \underset{\boldsymbol{u}^{all}_D}{\arg \min }\norm{\boldsymbol{u}^{all}_D}^{2} \nonumber \\
    \text{subject to} \quad \mathcal{A}^H\boldsymbol{u}^{all}_D \leq \boldsymbol{b}^H  
\end{align}
Here $\boldsymbol{u}^{*all}_D$ are the optimal velocities for all the dog robots to ensure both defending $\mathcal{P}$ and collision avoidance simultaneously.   By construction, our approach is centralized \textit{i.e.} it computes velocities of all dog robots together. Future work will consider ways to decentralize this approach.  \\
\textbf{Considering multiple protected zones:} While in the above derivation, we considered preventing the sheep from breaching only one protected zone, we can just as easily consider another protected zone by formulating similar constraints $\mathcal{A}^H_2\boldsymbol{u}^{all}_D \leq \boldsymbol{b}^H_2$ on the dogs' velocities. By augmenting  \eqref{dog_control} with these constraints for the other zone, we will be able to defend both zones from all sheep simultaneously. This compositionality is a benefit offered by our constraint based framework. An experimental validation of this is shown in Fig. \ref{fig:defendingprotectedzon3}. In the following discussion, we prove that for the one dog v/s one sheep case, \eqref{dog_control} is always feasible:
\begin{theorem}
\label{theorem_caseA}
If there is one dog and one sheep, then \eqref{dog_control} always has a solution.
\end{theorem}
\begin{proof}
Let the position of the dog be $\boldsymbol{x}_D$ and that of the sheep be $\boldsymbol{x}_S$. The sheep dynamics can be simplified to 
\begin{align}
    \label{eq:dyn11}
    \dot{\boldsymbol{x}}_S = \boldsymbol{f}(\boldsymbol{x}_{S},\boldsymbol{x}_{D})= k_{G}\left(\boldsymbol{x}_{G}-\boldsymbol{x}_{S}\right)
+k_{D}  \frac{\boldsymbol{x}_{S} - \boldsymbol{x}_{D}}{\norm{\boldsymbol{x}_{S} - \boldsymbol{x}_{D}}^{3}}
\end{align}
The only case when \eqref{dog_control} does not have a solution is when the defending constraint is infeasible \textit{i.e.} when $A^H\boldsymbol{u}_D \leq b^H$ is infeasible. This can occur when 
\begin{itemize}
    \item  either when $A^H = \boldsymbol{0}$ and $b^H<0$ (\textcolor{blue}{possibility 1})
    \item or when $b^H = -\infty$ (\textcolor{blue}{possibility 2}).
\end{itemize} 
For this case $A^H$ is:
\begin{align}
    A^H = (\boldsymbol{x}_{P} -\boldsymbol{x}_{S})^T
     \mathbb{J}_{11}^D
\end{align}
Thus, if $\mathbb{J}_{11}^D$ is non-singular, $(\boldsymbol{x}_{P} -\boldsymbol{x}_{S})^T
     \mathbb{J}_{11}^D \neq \boldsymbol{0}$. From our calculations, we find that the determinant of $\mathbb{J}_{11}^D$ is
\begin{align}
    det(\mathbb{J}_{11}^D) = \frac{-2k_D^2}{\norm{ \boldsymbol{x}_{D} -\boldsymbol{x}_{S}}^3} 
\end{align}
As long as the distance between the dog and the sheep is finite, $det(\mathbb{J}_{11}^D)$ is always non zero. Thus, there exists no null space for the jacobian matrix $\mathbb{J}_{11}^D$. This implies $A^H \neq \boldsymbol{0}$ $\forall \boldsymbol{x}_S \in \mathbb{R}^n, \boldsymbol{x}_D \in \mathbb{R}^2$. This rules out \textcolor{blue}{possibility 1} for infeasibility. For \textcolor{blue}{possibility 2}, we need to examine when does $b^H \longrightarrow -\infty$. The expression for $b^H$ is:
\begin{align*}
      b^H = \boldsymbol{f}^T\boldsymbol{f}
      + (\boldsymbol{x}_{S} - \boldsymbol{x}_{P})^T\mathbb{J}_{11}^S \boldsymbol{f}
      + \alpha(\boldsymbol{x}_{S}-\boldsymbol{x}_{P})^T\boldsymbol{f} + \beta\frac{h}{2}
\end{align*}
We want to find the worst case lower bound of $b^H$. Here $ \boldsymbol{f}^T\boldsymbol{f} \geq 0$ always. We assume that at the current time step, the sheep is outside the $\mathcal{P}$, this ensures $\beta \frac{h}{2} \geq 0$.
\begin{assumption}
   Assume that the following bounds hold $\norm{\boldsymbol{x}_S-\boldsymbol{x}_G} \leq M_1, \norm{\boldsymbol{x}_S-\boldsymbol{x}_P} \leq M_2$ and $\norm{\boldsymbol{x}_S-\boldsymbol{x}_D} \geq M_3$.
\end{assumption}
With these assumptions, we can lower bound $b^H$ as follows: 
\begin{align}
    b^H \geq (\boldsymbol{x}_{S} - \boldsymbol{x}_{P})^T\mathbb{J}_{11}^S \boldsymbol{f}
      + \alpha(\boldsymbol{x}_{S}-\boldsymbol{x}_{P})^T\boldsymbol{f} \nonumber \\
      \geq -(\sigma_{max}(\mathbb{J}_{11})+\alpha)\norm{\boldsymbol{f}}\norm{\boldsymbol{x}_S-\boldsymbol{x}_P} \nonumber \\
       \geq -(\sigma_{F}(\mathbb{J}_{11})+\alpha)\norm{\boldsymbol{f}}\norm{\boldsymbol{x}_S-\boldsymbol{x}_P}
\end{align}
Here $\norm{\boldsymbol{f}} \leq k_G\norm{\boldsymbol{x}_S-\boldsymbol{x}_G} + \frac{k_D}{\norm{\boldsymbol{x}_S-\boldsymbol{x}_D}^2}$ using triangle inequality on \eqref{eq:dyn11}. This gives   $\norm{\boldsymbol{x}_S-\boldsymbol{x}_P}\norm{\boldsymbol{f}} \leq k_GM_1M_2 + \frac{k_DM_2}{M_3^2}$.  We can show that $\sigma_{F}(\mathbb{J}_{11})  \leq \lambda_M \coloneqq  \sqrt{2k_G^2 + 5\frac{k^2_D}{M^6_3} + \frac{2k_Gk_D}{M_3^2}}$. 
Thus, using this, we obtain the following lower bound for $b^H$
\begin{align}
    b^H         \geq -(\lambda_M+\alpha)\bigg(k_GM_1M_2 + \frac{k_DM_2}{M_3^2} \bigg)
\end{align}
This shows that $b^H$ is lower bounded and thus does not reach $-\infty$. Hence \textcolor{blue}{possibility 2} is also ruled out. Thus, \eqref{dog_control} is always feasible.
\end{proof}
\subsection{Incorporating collision avoidance constraints}
The defending constraints we posed above do not guarantee that the dog robots won't collide with the sheep. Even though the sheep dynamics have repulsions from the dogs, the velocities computed can result in aggressive behavior. Thus, we augment the defending constraints with additional constraints to ensure collision free behavior.  Following the approach in \cite{wang2017safety}, we define a pairwise safety index $b^{ik}(\cdot): \mathbb{R}^2 \times \mathbb{R}^{2m} \longrightarrow \mathbb{R}$ as:
\begin{align*}
    b^{ik}(\boldsymbol{x}_{S_i} ,\boldsymbol{x}_{D_1},\cdots,\boldsymbol{x}_{D_m}) &= \norm{\boldsymbol{x}_{S_i} - \boldsymbol{x}_{D_k}}^2 - R_S^2 \nonumber \\
    &=\norm{\boldsymbol{x}_{S_i} - C_k\boldsymbol{x}^{all}_D}^2 - R_S^2
\end{align*}
 $ b^{ik}(\cdot) \geq 0$ iff dog $k$ is atleast $R_S$ distance away from sheep $i$. Here $C_k$ is a matrix defined appropriately to extract the position of the $k^{th}$ dog from $\boldsymbol{x}^{all}_D$.  If $b^{ik}(\boldsymbol{x}_{S_i}(0) ,\boldsymbol{x}^{all}_D(0))\geq 0$ $\forall k \in \{1,2,\cdots,m\}$, we would like to ensure that $b^{ik}(\boldsymbol{x}_{S_i}(t) ,\boldsymbol{x}^{all}_{D}(t)) \geq 0$ $\forall t \geq 0$ and $\forall k \in \{1,2,\cdots,m\}$. This can be achieved by requiring that
\begin{align}
     \Dot{b}^{ik}(\boldsymbol{x}) + \gamma b^{ik}(\boldsymbol{x}) \geq 0 \hspace{0.5cm} \forall k \in \{1,2,\cdots,m\}
\end{align}
where $\gamma > 0$. This gives us a total of $m$ linear constraints on the velocity of the dog robots for avoiding collisions with the $i^{th}$ sheep:
\begin{align}
    \label{collisioncon1}
    A^C_i \boldsymbol{u}^{all}_D\leq b^C_i 
\end{align}
where,
\begin{align}
    A^C_i &=  
     \left[\begin{matrix}
     (\boldsymbol{x}_{S_i} -\boldsymbol{x}_{D_1})^TC_1 \\ \vdots \\ (\boldsymbol{x}_{S_i} -\boldsymbol{x}_{D_m})^TC_m 
   \end{matrix}\right]  \nonumber \\
      b^C_i &= \left[\begin{matrix}\frac{\gamma}{2} b^{i1} + (\boldsymbol{x}_{S_i} -\boldsymbol{x}_{D_1})^T \boldsymbol{f}_{i} \\ \vdots \\ \frac{\gamma}{2} b^{im} + (\boldsymbol{x}_{S_i} -\boldsymbol{x}_{D_m})^T \boldsymbol{f}_{i}\end{matrix}\right] 
\end{align}
To ensure all collision avoidance with all $n$ sheep\footnote{inter-dog collision avoidance constraints can also be added following a similar procedure.},  we compose constraints \eqref{collisioncon1} for all the herd as follows:
\begin{align}
\label{allcol}
    \left[\begin{matrix}
A^C_1 \\
\vdots \\
A^C_n
\end{matrix}\right] \boldsymbol{u}^{all}_D \leq \left[\begin{matrix}
b^C_1 \\
\vdots \\
b^C_n
\end{matrix}\right] \implies \mathcal{A}^C\boldsymbol{u}^{all}_D \leq \boldsymbol{b}^C
\end{align}
Given the defending \eqref{allherd} and collision avoidance \eqref{allcol} constraints on the dogs' velocities, we compose them together using the following QP:
\begin{align}
\label{dog_control_with_col}
    \boldsymbol{u}^{*all}_D = \underset{\boldsymbol{u}^{all}_D}{\arg \min }\norm{\boldsymbol{u}^{all}_D}^{2}\\
    \text{subject to} \quad \mathcal{A}^H\boldsymbol{u}^{all}_D \leq \boldsymbol{b}^H    \notag\\        
\mathcal{A}^C\boldsymbol{u}^{all}_D \leq \boldsymbol{b}^C  \notag
\end{align}

Here $\boldsymbol{u}^{*all}_D$ are the optimal velocities for all the dog robots to ensure both defending and collision avoidance simultaneously. The cost function penalizes the total speed of the dog robots, thus encouraging them to minimize their movement.  

\section{Results}
\label{results}
In this section, we show results of our approach by testing it on different scenarios consisting of varying numbers of sheep and dog and varying their initial positions. Additionally, we also run validate these results experimentally. We perform several experiments with nonholonomic Khepera robots and demonstrate how our algorithm find velocities for one dog to simultaneously defend multiple protected zones from  multiple sheep.
\begin{figure*}
	\centering     %%% not \center
	\subfigure[Three dog robots v/s three sheep robots. ]{\label{fig:dA0}\includegraphics[trim={1.1cm 1.5cm 1.5cm 1.5cm},clip,width=0.66\columnwidth]{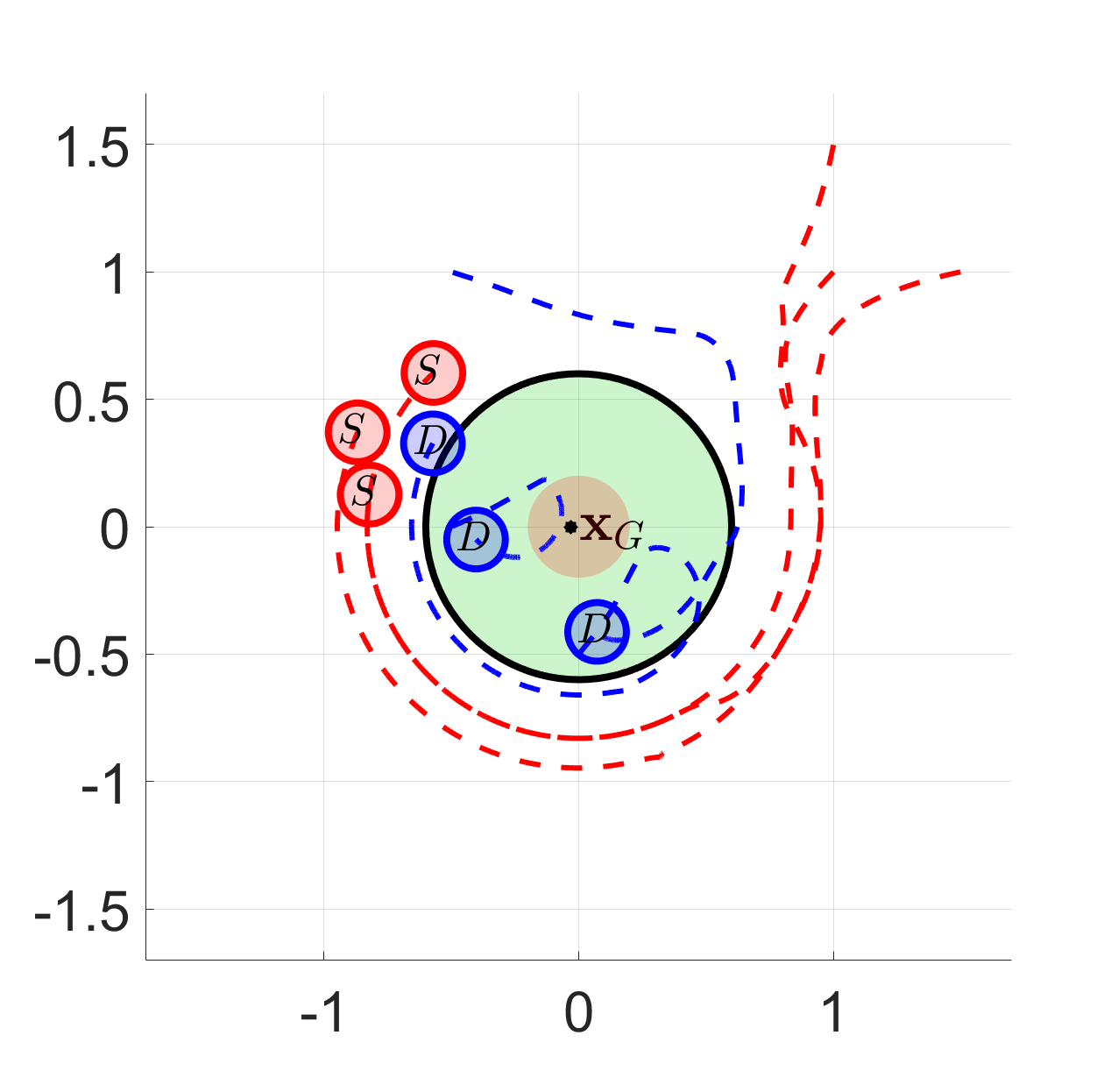}}
	\subfigure[Three dog robots v/s five sheep. ]{\label{fig:dB0}\includegraphics[trim={1.1cm 1.5cm 1.5cm 1.5cm},clip,width=0.66\columnwidth]{Images/3v5withcollision.png}}
	\subfigure[Three dog robots v/s three sheep robots. ]{\label{fig:dC0}\includegraphics[trim={1.1cm 1.5cm 1.5cm 1.5cm},clip,width=.66\columnwidth]{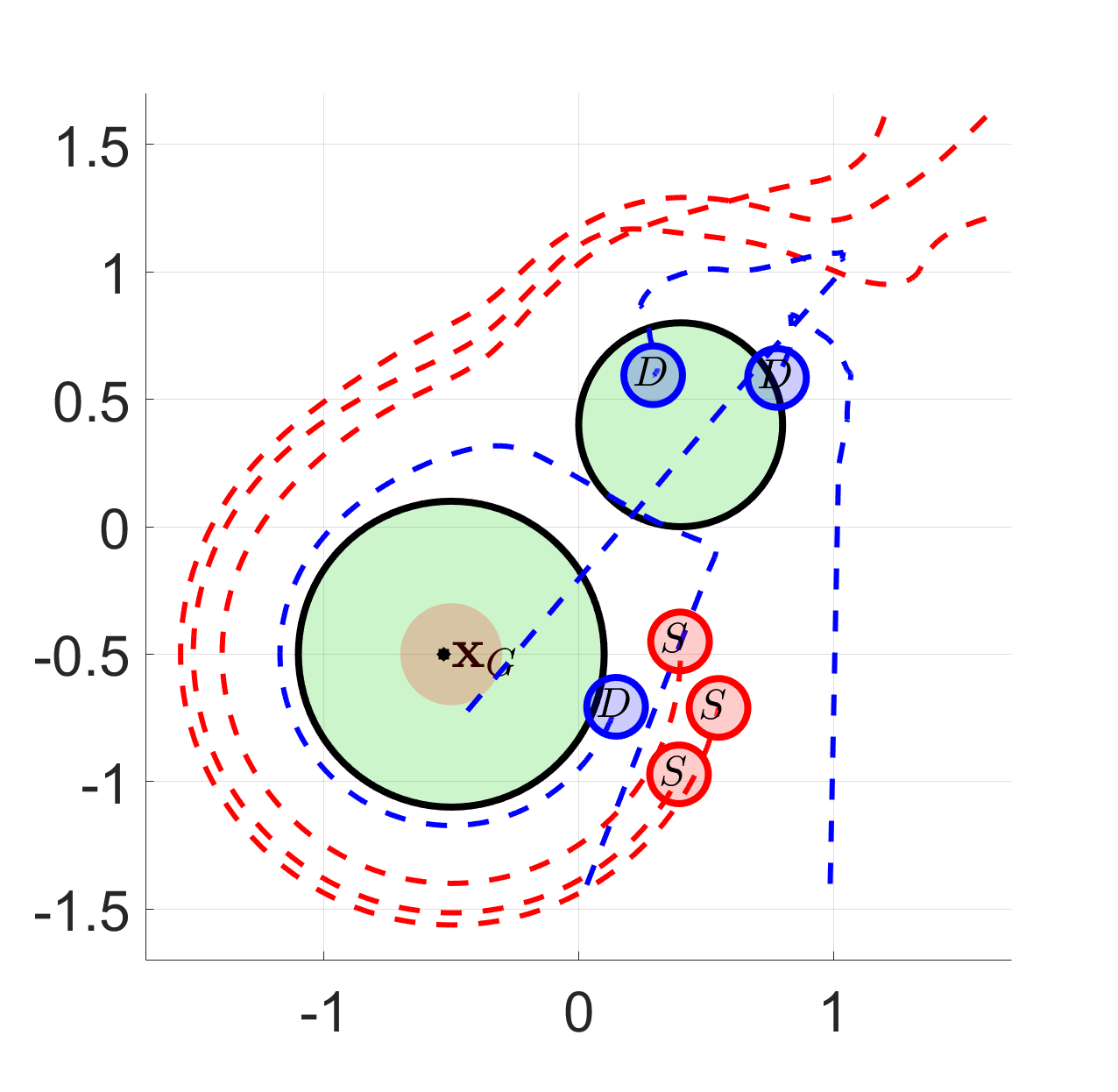}}
	\caption{Preventing the breaching of the protected zone. In these simulations, the dog is shown in blue and the sheep is shown in red. The green disc represents the protected zone. The nominal task of the red agent is to go straight towards its goal $\boldsymbol{x}_G$. However, since this would result in infiltration of the protected zone, the dog intervenes using the control algorithm presented in \eqref{dog_control_with_col}. In \ref{fig:dC0}, we defend two protected zones from three sheep.}
	\label{fig:defendingprotectedzon}
\end{figure*}

\begin{figure*}
	\centering     %%% not \center
	\subfigure[$t = 0s$ ]{\label{fig:dA1}\includegraphics[trim={0.0cm 0cm 0cm 0cm},clip,width=0.495\columnwidth]{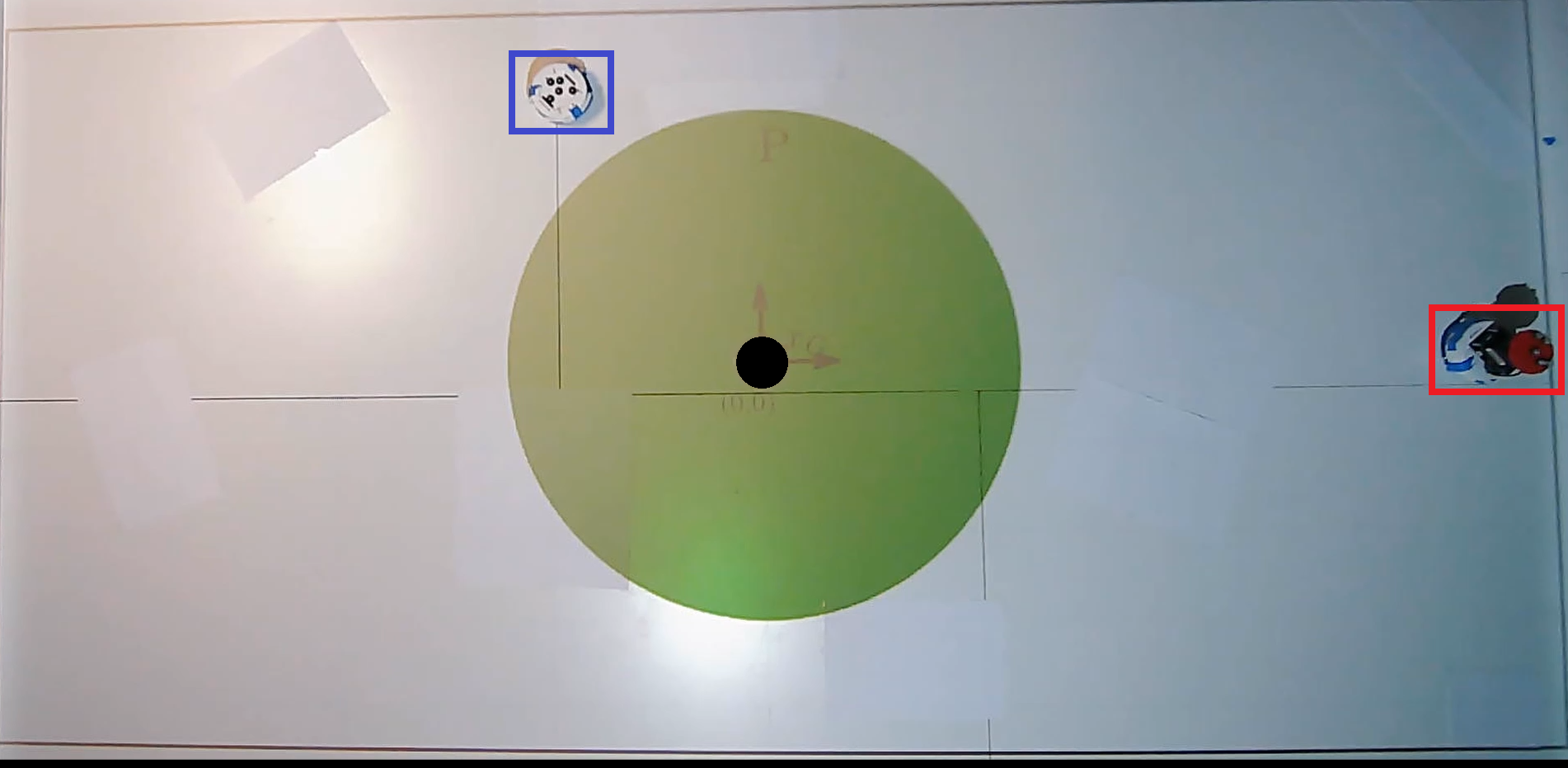}}
	\subfigure[$t = 3s$ ]{\label{fig:dB1}\includegraphics[trim={0cm 0cm 0cm 0cm},clip,width=0.495\columnwidth]{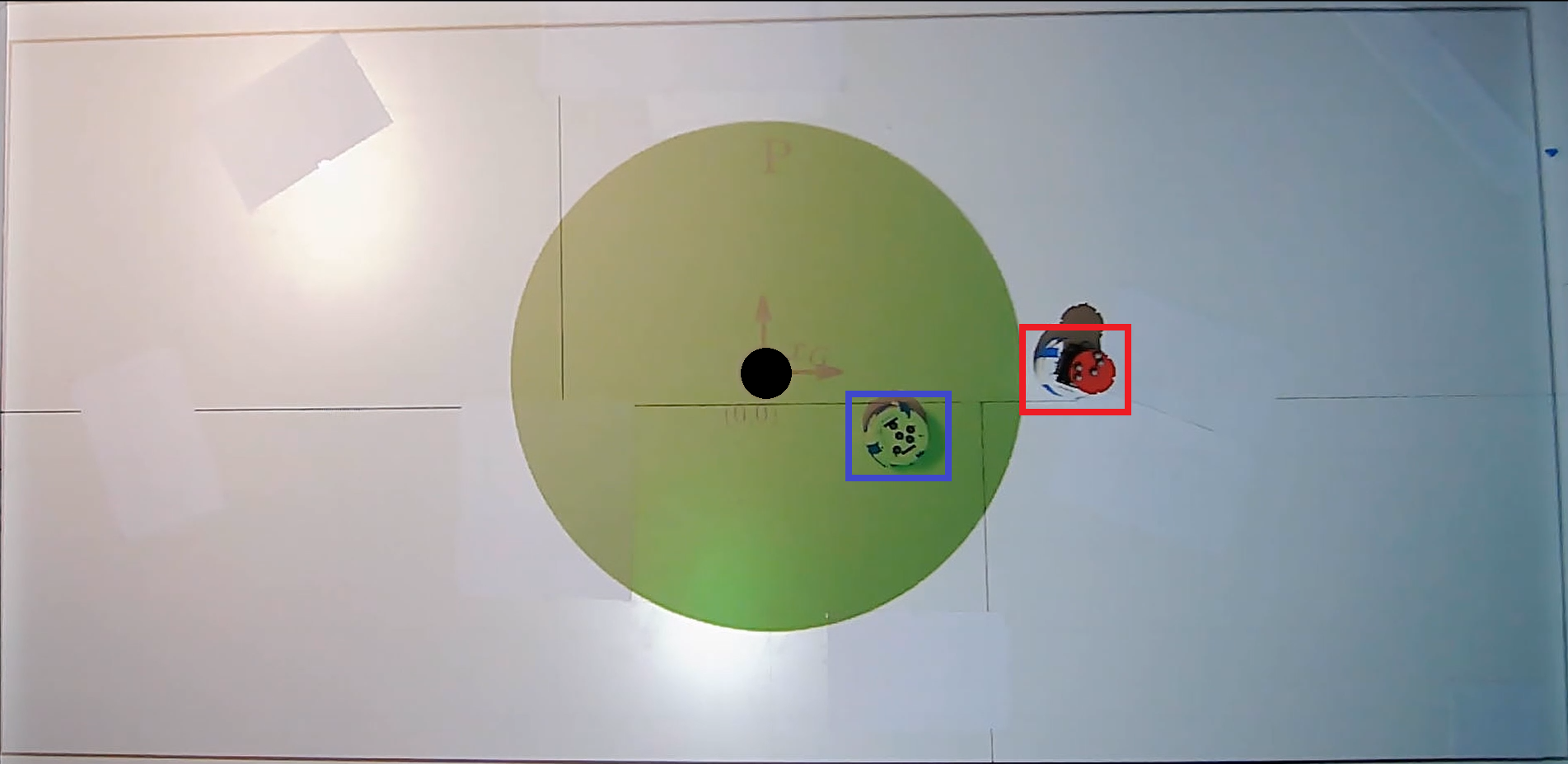}}
	\subfigure[$t = 20s$ ]{\label{fig:dC1}\includegraphics[trim={0cm 0cm 0cm 0cm},clip,width=.495\columnwidth]{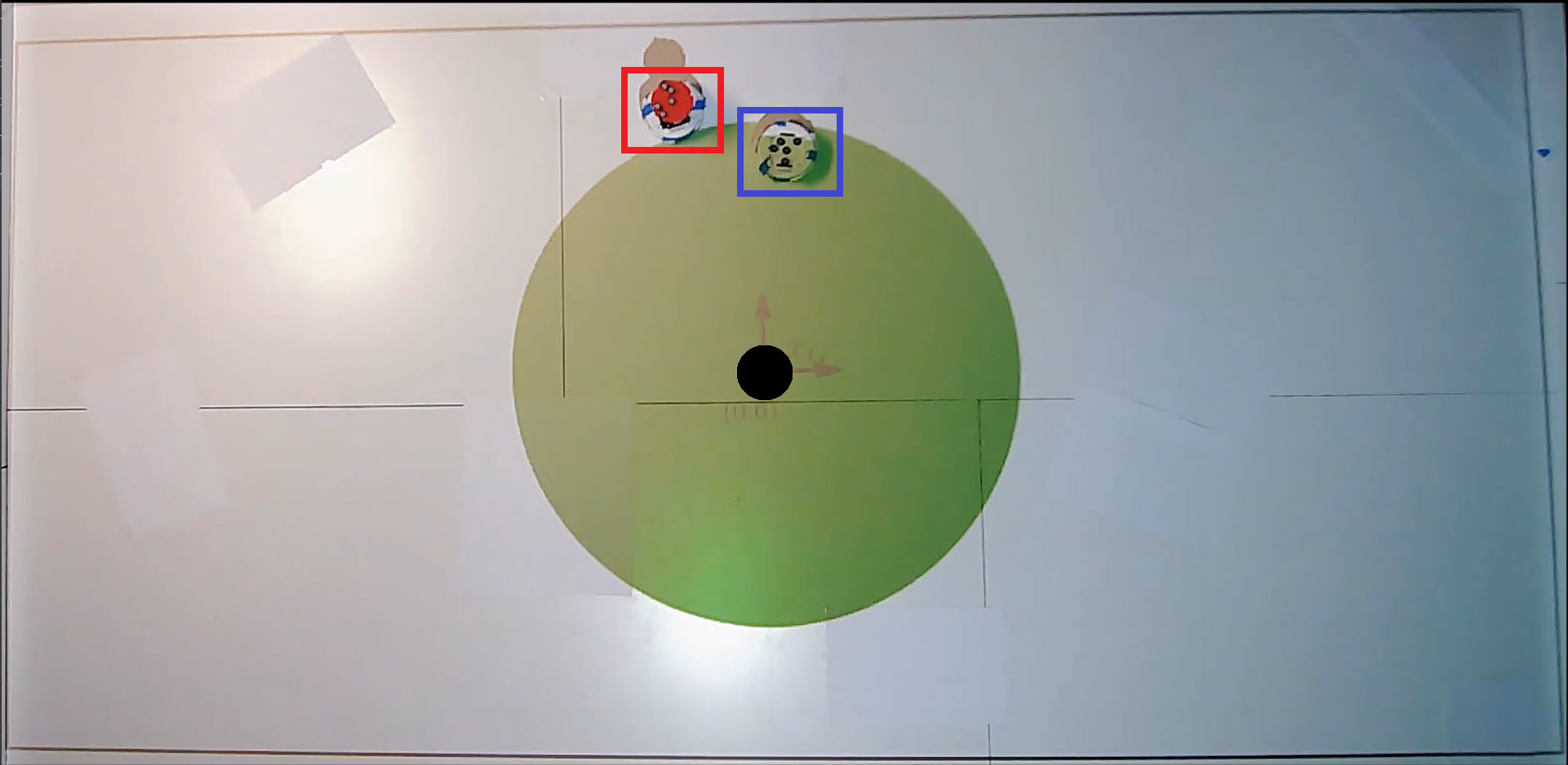}}
		\subfigure[$t = 55s$ ]{\label{fig:dD1}\includegraphics[trim={0cm 0cm 0cm 0cm},clip,width=.495\columnwidth]{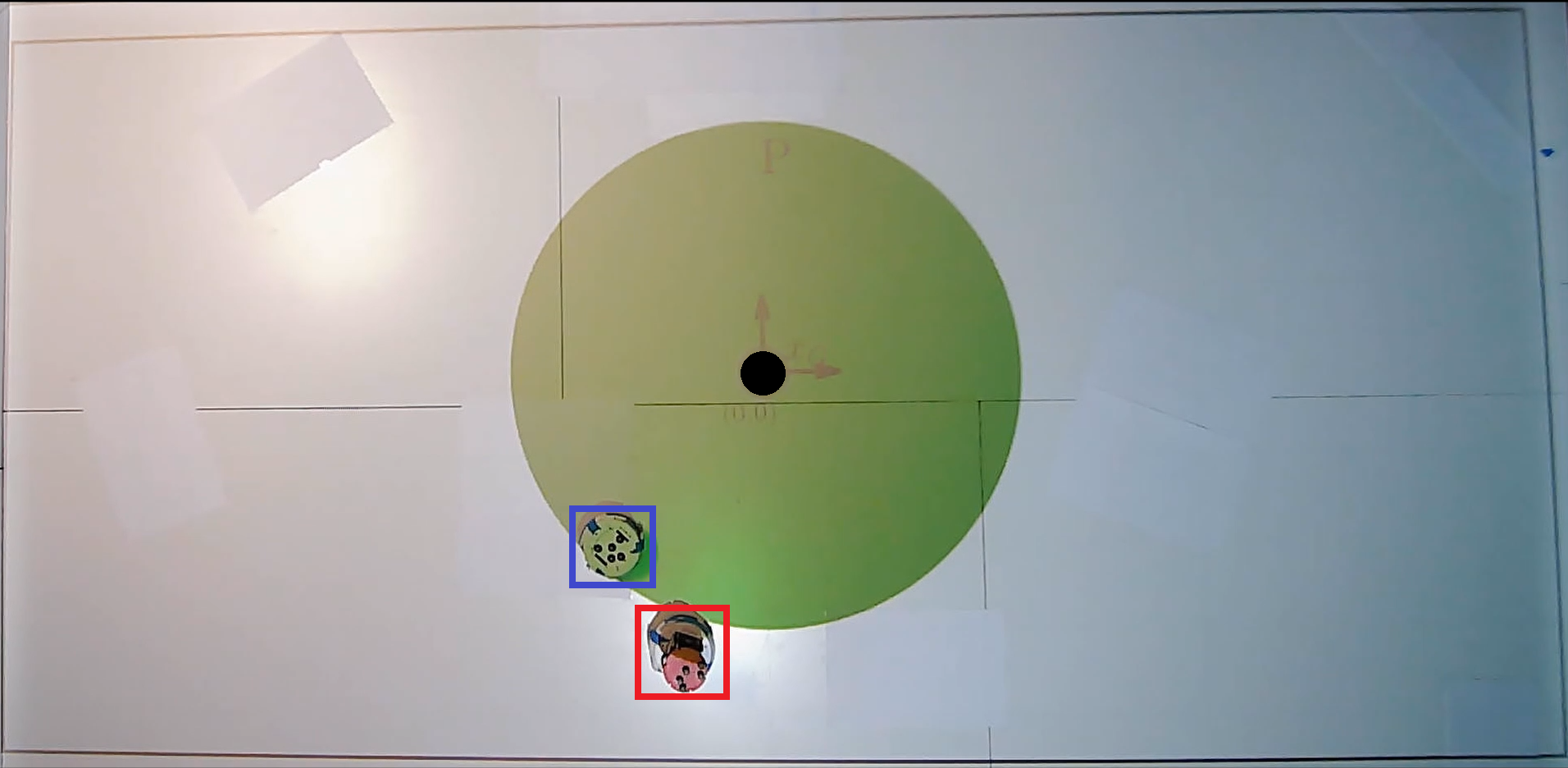}}
	\caption{Hardware experiment with one dog robot preventing one sheep from the breaching of the protected zone. The dog robot is highlighted in blue and the sheep in red. The goal position $x_G$ is at the center of the protected zone and given as a black solid circle. The nominal task of the sheep is to go straight towards its goal $\boldsymbol{x}_G$. However, since this would result in infiltration of the protected zone, the dog intervenes using the control algorithm presented in \eqref{dog_control}. Video at \url{https://tinyurl.com/2p9fjeft}. }
	\label{fig:defendingprotectedzon1}
\end{figure*}

\begin{figure*}
	\centering     %%% not \center
	\subfigure[$t = 0s$ ]{\label{fig:dA2}\includegraphics[trim={0.0cm 0cm 0cm 0cm},clip,width=0.495\columnwidth]{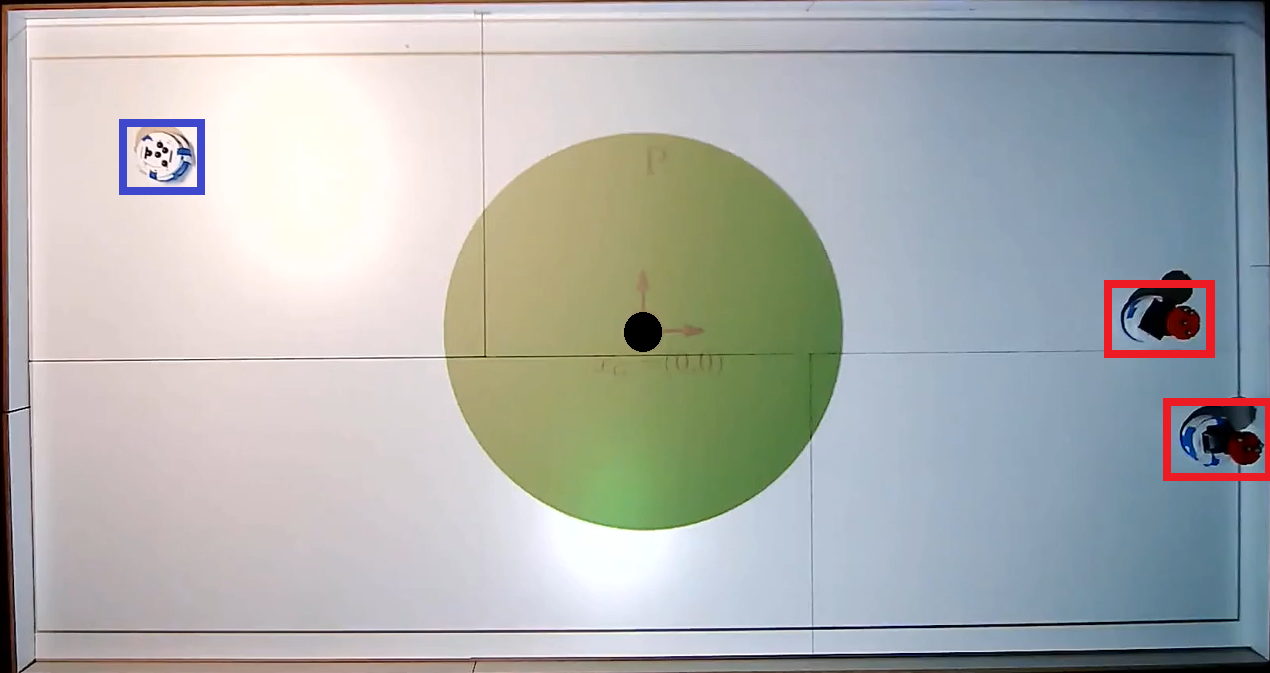}}
	\subfigure[$t = 6s$ ]{\label{fig:dB2}\includegraphics[trim={0cm 0cm 0cm 0cm},clip,width=0.495\columnwidth]{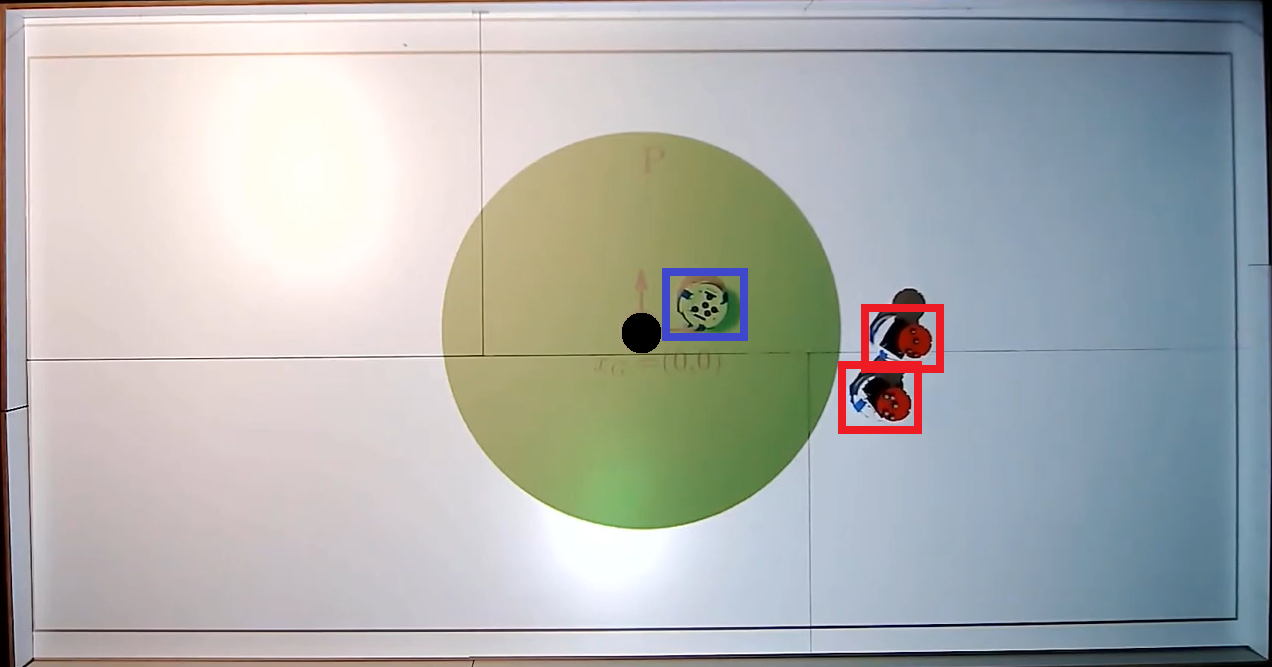}}
	\subfigure[$t = 26s$]{\label{fig:dC2}\includegraphics[trim={0cm 0cm 0cm 0cm},clip,width=.495\columnwidth]{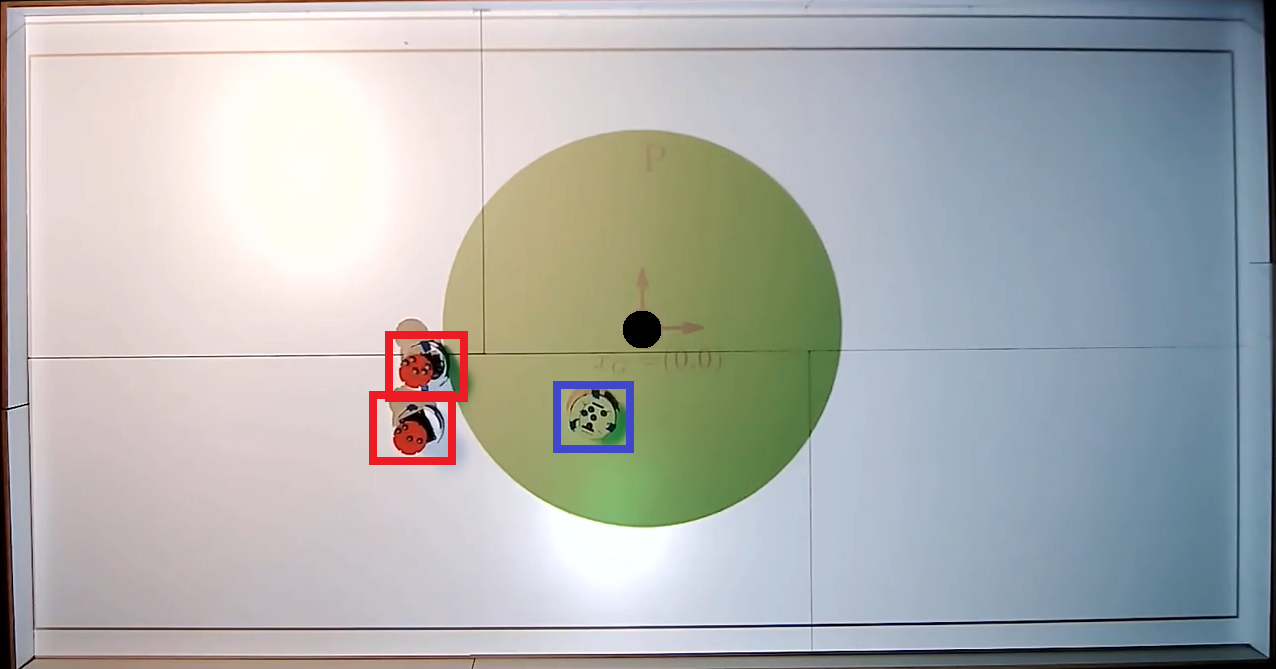}}
		\subfigure[$t = 35s$ ]{\label{fig:dD2}\includegraphics[trim={0cm 0cm 0cm 0cm},clip,width=.495\columnwidth]{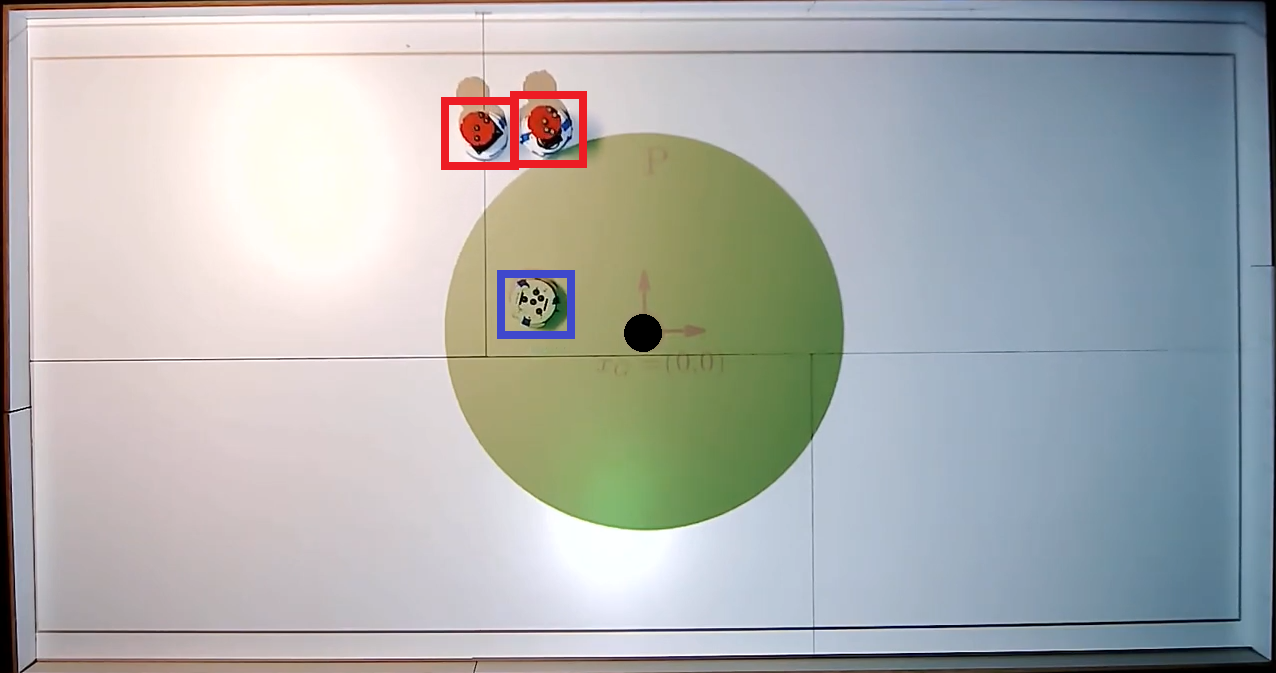}}
	\caption{Hardware experiment with one dog robot preventing two sheep from the breaching of the protected zone. The dog robot is highlighted with a blue box and sheep using a red box. The goal position $\boldsymbol{x}_G$ is at the center of the protected zone and shown as a solid black dot. Video at \url{https://tinyurl.com/37rduh43}.}
	\label{fig:defendingprotectedzon2}
\end{figure*}

\begin{figure*}
	\centering     %%% not \center
	\subfigure[$t = 0s$ ]{\label{fig:dA3}\includegraphics[trim={0.0cm 0cm 0cm 0cm},clip,width=0.495\columnwidth]{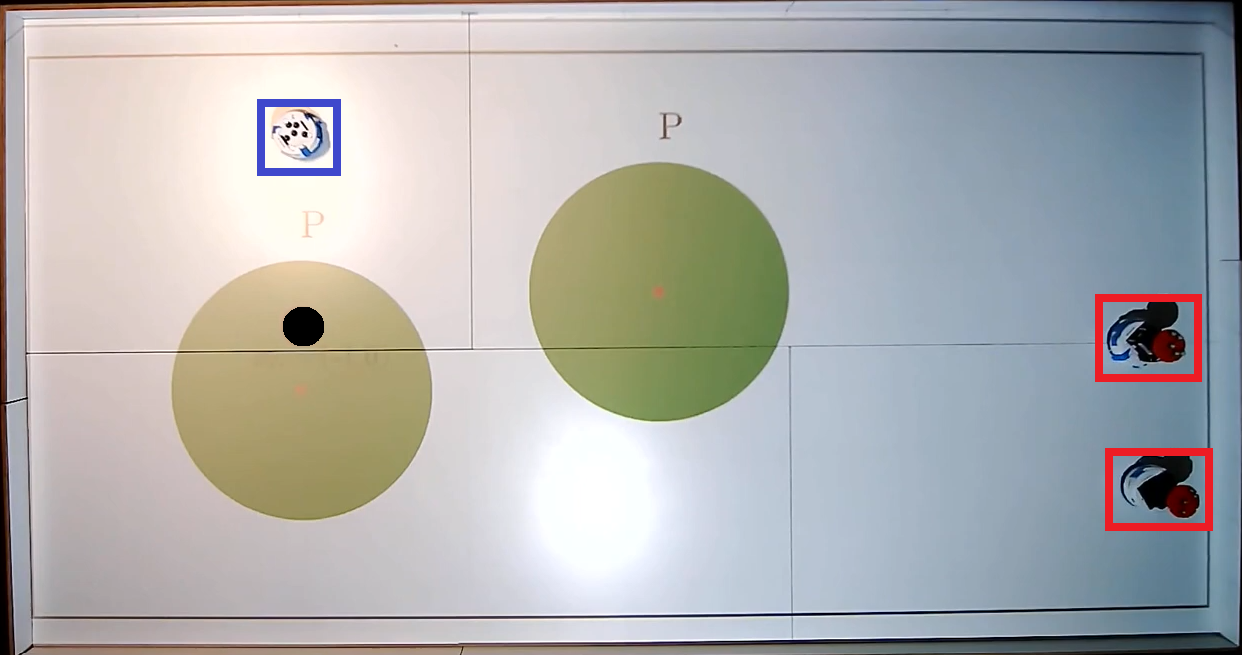}}
	\subfigure[$t = 7s$ ]{\label{fig:dB3}\includegraphics[trim={0cm 0cm 0cm 0cm},clip,width=0.495\columnwidth]{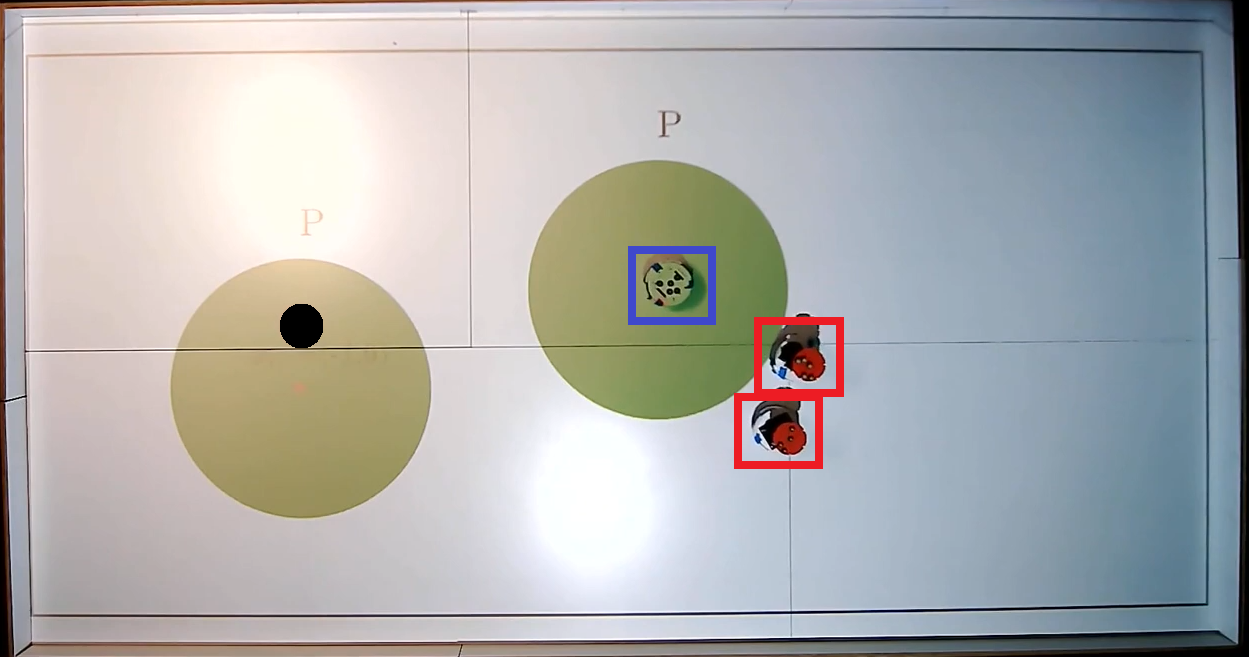}}
	\subfigure[$t = 10s$ ]{\label{fig:dC3}\includegraphics[trim={0cm 0cm 0cm 0cm},clip,width=.495\columnwidth]{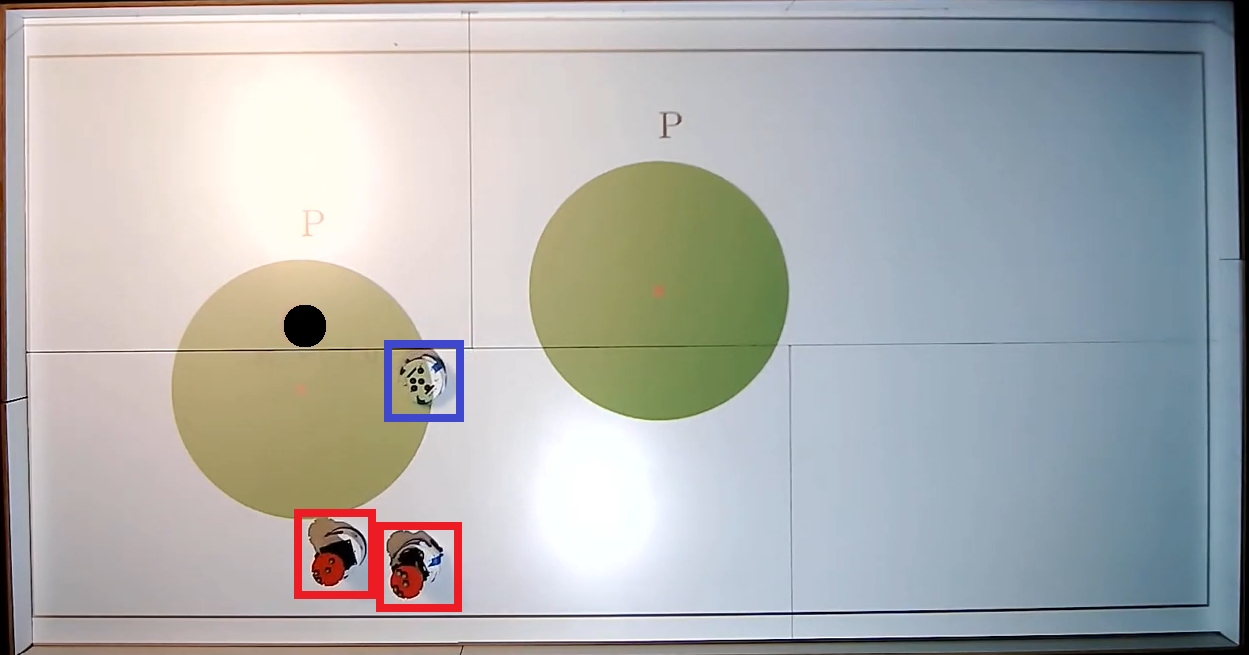}}
	\subfigure[$t = 14s$ ]{\label{fig:dC3}\includegraphics[trim={0cm 0cm 0cm 0cm},clip,width=.495\columnwidth]{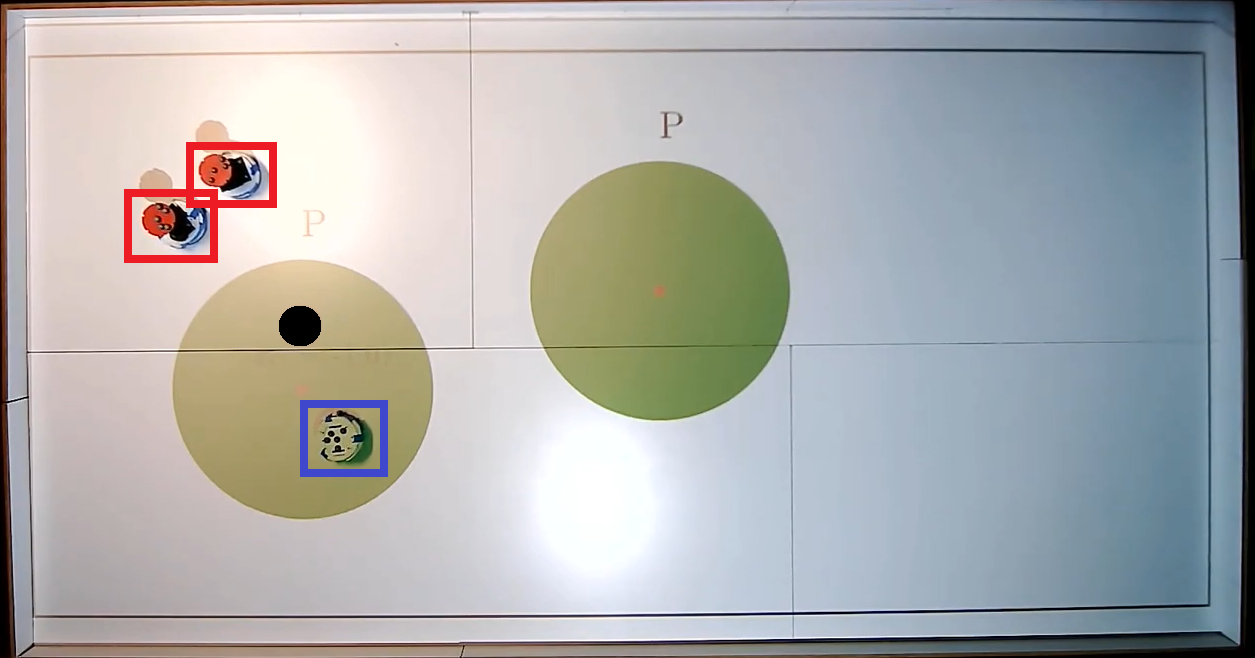}}
	\caption{Hardware experiment with one dog robot preventing two sheep from the breaching of two protected zones. The goal lies in the left most protected zone. Video at \url{ https://tinyurl.com/ycuyhwe6}. }
	\label{fig:defendingprotectedzon3}
\end{figure*}

\subsection{Numerical Simulation}
We represent the protected zone using a circular disc with radius $R_p$ and its center at the origin \textit{i.e.} $\boldsymbol{x}_P = \boldsymbol{0}$. In our simulations, we purposefully choose the agent's goal  $\boldsymbol{x}_{G} = \boldsymbol{x}_P$ so that the sheep are motivated to breach the protected zone should the dog robots not interfere. Thus, this is an adversarial scenario. The initial position $\boldsymbol{x}_{S_i}(0)$ of all the sheep are chosen such that they are all close to each other. This is done to ensure that the sheep have enough time to stabilize/cohese as a flock before interacting with the dog robots. The initial position $\boldsymbol{x}^{all}_D(0)$ of the dog robots are chosen randomly within the area of operation. The sheep's velocities are calculated using \eqref{sheepdynamics}. The values of the gains in the sheep dynamics were taken as $k_G = 1$, $k_S = 0.3$ and $k_D = 0.08$.

The velocities of the dog robot was obtained using eqn. \eqref{dog_control_with_col}. The hyperparameters $\alpha,\beta,\gamma$ are tuned satisfy the conditions on the design parameters (\ref{eq:condition_on_p1}, \ref{eq:condition_on_p2}) . Figure \ref{fig:defendingprotectedzon} shows three simulation results for this behavior. In these simulations, we varied the initial position of the sheep (blue), the dog (red), number of sheep and the number of dogs. As can be noticed from the pictures, in all three scenarios, the dogs robots able to successfully intercept the sheep and prevent them from entering the protected zone while also avoiding collision with the sheep.
\subsection{Monte Carlo Simulations}
We further study the performance of the proposed control strategy by using Monte Carlo simulations with varying initial configurations  and varying number of sheep $n$ and dog robots $m$. The values of the constants in sheep dynamics are $k_G = 1$, $k_S = 0.3$ and $k_D = 0.08$. We vary $n$ and $m$ from 1 to 10 and for a given pair of ($n,m$) we run the simulation for a hundred times  with a random initialization of $\boldsymbol{x}_0=(\boldsymbol{x}_{S_1}(0),\cdots,\boldsymbol{x}_{S_n}(0),\boldsymbol{x}_{D_1}(0),\cdots,\boldsymbol{x}_{D_m}(0))$ in every run.
Table \ref{tab:A_without_col} reports these results. Each entry of this table reports the percentage success rate \textit{i.e.} in how many cases the sheep were diverted away from the protected zone. As can be seen, almost all entries are 100, which proves the success of our algorithm. 
\begin{table}[ht]
\renewcommand{\arraystretch}{1}
\caption{Performance of the proposed strategy with varying number of sheep and dog robots. Here, we did not consider collision avoidance constraints \textit{i.e.} the dogs were allowed to run into the sheep. } \label{tab:A_without_col}
\centering
	\begin{tabular}{|m{0.09\textwidth}|m{0.05\textwidth}|m{0.05\textwidth}|m{0.05\textwidth}|m{0.05\textwidth}|m{0.04\textwidth}|}
	\hline
\backslashbox{$N_{S}$}{$N_{D}$} & \textbf{2} & \textbf{4} & \textbf{6} & \textbf{8} & \textbf{10}
\\ \hline
\textbf{2}   &   100  &     100  &   100  &   100  &   100 \\
    \hline
\textbf{4}   &   100  &    100  &   100  &   100  &   100 \\
    \hline
\textbf{6}   &   100  &     98 &  100  &  100  &  100\\
    \hline
\textbf{8}   &  100  &     98  &  100  &  100  &   98 \\
    \hline
\textbf{10}   &   100  &     98 &   98  &   100  &   96 \\
    \hline
\end{tabular}
\end{table}
Further, we considered the impact of including collision avoidance constraints. These results are reported in Table \ref{tab:A_with_col}. Because of additional constraints, it is possible that collision avoidance conflicts with the defending constraint. As a result, we do not observe as good successes in this case compared to when there are no collision avoidance constraints.
\begin{table}[ht]
\renewcommand{\arraystretch}{1}
\caption{Performance of the proposed strategy with varying number of sheep and dog robots. Here we considered collision avoidance constraints in the dynamics of the dogs. } \label{tab:A_with_col}
\centering
	\begin{tabular}{|m{0.09\textwidth}|m{0.05\textwidth}|m{0.05\textwidth}|m{0.05\textwidth}|m{0.05\textwidth}|m{0.04\textwidth}|}
	\hline
\backslashbox{$N_{S}$}{$N_{D}$} & \textbf{2} & \textbf{4} & \textbf{6} & \textbf{8} & \textbf{10}
\\ \hline
\textbf{2}   &   72  &     99  &   99  &   100  &   100 \\
    \hline
\textbf{4}   &   62  &     74  &   90  &   97  &   100 \\
    \hline
\textbf{6}   &   28  &     83  &  99  &   99  &  100\\
    \hline
\textbf{8}   &  63  &     82  &   100  &   100  &   100 \\
    \hline
\textbf{10}   &   70  &     79  &   90  &   91  &   94 \\
    \hline
\end{tabular}
\end{table}
\subsection{Hardware Experiments}
Finally, we tested our algorithm in robots in the multirobot test arena in our lab. It consists of a 14ft $\times$ 7ft platform, several Khepera IV robots and additionally eight Vicon cameras for motion tracking. All control inputs are computed on a desktop and conveyed to the robots over WiFi.  While we developed our algorithms assuming that the dynamics of all agents are single-integrator based, the robots have unicycle dynamics given by 
\begin{align}
    \begin{pmatrix}
    \dot{x} \\
    \dot{y}\\
    \dot{\theta}
    \end{pmatrix} = \begin{pmatrix}
    v\cos{\theta} \\
    v\sin{\theta}\\
    \omega
    \end{pmatrix}
\end{align}
Thus, we do a minor adjustment to map the inputs computed from our algorithms to the angular speed and forward translational speed of these robots. This is done by considering a point at a distance $d$ on the $x_b$ axis of the body frame of the robot:
\begin{align}
\label{omega}
    &\boldsymbol{x} = \begin{pmatrix}
    x + d\cos{\theta} \\
    y + d\sin{\theta}
    \end{pmatrix} \nonumber \\
    \implies &\dot{\boldsymbol{x}} = \underbrace{\begin{pmatrix}
    \cos{\theta} &-\sin{\theta} \\ 
    \sin{\theta} &\cos{\theta}
    \end{pmatrix} \begin{pmatrix}
    1 &0 \\ 
    0 &d
    \end{pmatrix}}_{M}\begin{pmatrix}
    v \\ 
    \omega
    \end{pmatrix} = \tilde{\boldsymbol{u}} \nonumber \\
    \implies &\begin{pmatrix}
    v \\ 
    \omega
    \end{pmatrix} = M^{-1}\tilde{\boldsymbol{u}}
\end{align}
For the robots representing the sheep,  $\tilde{\boldsymbol{u}}$ is obtained from \eqref{sheepdynamics} while for the robots representing the dog,  $\tilde{\boldsymbol{u}}$ is obtained from \eqref{dog_control}. In Fig. \ref{fig:defendingprotectedzon1}, we have one sheep (in red box) and one dog robot (in blue box). The protected zone is highlighted in green and the goal and center of the protected zone are the black dot. We use \eqref{dog_control} to compute the velocity of the dog robot and convert it to angular speed and forward translational speed using \eqref{omega}. As can be noted from the snapshots, the dog robot is able successfully  defend the zone from the sheep. Next we consider multiple sheep in Fig. \ref{fig:defendingprotectedzon2}. As can be seen from the snapshots, in this case, the dog is able to defend the zone from both sheep. Finally, in  Fig. \ref{fig:defendingprotectedzon3} we demonstrate that our approach is compositional \textit{i.e.} we can have multiple protected zones. In this figure, we purposefully kept the goal of the sheep in the left most protected zone. This way, the sheep would be incentivized to breach both the protected zones. Yet still, our algorithm is able to find velocities for dogs to defend both the zones from both sheep.

\section{Conclusions}
 \label{conclusions}
 
In this paper, we developed a novel optimization-based control strategy for a group of dog robots to prevent a herd of sheep agents from breaching  a protected zone. We have proven the feasibility of the algorithm for the single dog v/s single sheep case. Empirical results show that our designed controller can defend the protected zone from a flock of multiple sheep using multiple dogs as well. The results also show that the algorithm is composable and allows us to include multiple protected zones. Future work will focus on finding design parameters of the constraints such that the velocities computed by our controller does not exceed actuator limits. We also aim to perform hardware experiments with a higher number of sheep and dogs using Khepera robots. Further, in our current work, we assumed known dynamics of sheep. In future, we plan to extend this to the case where we learn their dynamics online while simultaneously performing defense of the protected zone. 

% \addtolength{\textheight}{-3cm}   % This command serves to balance the column lengths
                                  % on the last page of the document manually. It shortens
                                  % the textheight of the last page by a suitable amount.
                                  % This command does not take effect until the next page
                                  % so it should come on the page before the last. Make
                                  % sure that you do not shorten the textheight too much.

%%%%%%%%%%%%%%%%%%%%%%%%%%%%%%%%%%%%%%%%%%%%%%%%%%%%%%%%%%%%%%%%%%%%%%%%%%%%%%%%

%%%%%%%%%%%%%%%%%%%%%%%%%%%%%%%%%%%%%%%%%%%%%%%%%%%%%%%%%%%%%%%%%%%%%%%%%%%%%%%%

\bibliographystyle{IEEEtran}
% \bibliography{refs}
\bibliography{root}

% Generated by IEEEtran.bst, version: 1.14 (2015/08/26)
\begin{thebibliography}{10}
\providecommand{\url}[1]{#1}
\csname url@samestyle\endcsname
\providecommand{\newblock}{\relax}
\providecommand{\bibinfo}[2]{#2}
\providecommand{\BIBentrySTDinterwordspacing}{\spaceskip=0pt\relax}
\providecommand{\BIBentryALTinterwordstretchfactor}{4}
\providecommand{\BIBentryALTinterwordspacing}{\spaceskip=\fontdimen2\font plus
\BIBentryALTinterwordstretchfactor\fontdimen3\font minus
  \fontdimen4\font\relax}
\providecommand{\BIBforeignlanguage}[2]{{%
\expandafter\ifx\csname l@#1\endcsname\relax
\typeout{** WARNING: IEEEtran.bst: No hyphenation pattern has been}%
\typeout{** loaded for the language `#1'. Using the pattern for}%
\typeout{** the default language instead.}%
\else
\language=\csname l@#1\endcsname
\fi
#2}}
\providecommand{\BIBdecl}{\relax}
\BIBdecl

\bibitem{d2012guest}
R.~D'Andrea, ``Guest editorial: A revolution in the warehouse: A retrospective
  on kiva systems and the grand challenges ahead,'' \emph{IEEE Transactions on
  Automation Science and Engineering}, vol.~9, no.~4, pp. 638--639, 2012.

\bibitem{d2003distributed}
R.~D'Andrea and G.~E. Dullerud, ``Distributed control design for spatially
  interconnected systems,'' \emph{IEEE Transactions on automatic control},
  vol.~48, no.~9, pp. 1478--1495, 2003.

\bibitem{kazmi2011adaptive}
W.~Kazmi, M.~Bisgaard, F.~Garcia-Ruiz, K.~D. Hansen, and A.~la~Cour-Harbo,
  ``Adaptive surveying and early treatment of crops with a team of autonomous
  vehicles,'' in \emph{Proceedings of the 5th European Conference on Mobile
  Robots ECMR 2011}, 2011, pp. 253--258.

\bibitem{ji2007distributed}
M.~Ji and M.~Egerstedt, ``Distributed coordination control of multiagent
  systems while preserving connectedness,'' \emph{IEEE Transactions on
  Robotics}, vol.~23, no.~4, pp. 693--703, 2007.

\bibitem{lin2004multi}
J.~Lin, A.~S. Morse, and B.~D. Anderson, ``The multi-agent rendezvous
  problem-the asynchronous case,'' in \emph{2004 43rd IEEE Conference on
  Decision and Control (CDC)(IEEE Cat. No. 04CH37601)}, vol.~2.\hskip 1em plus
  0.5em minus 0.4em\relax IEEE, 2004, pp. 1926--1931.

\bibitem{reynolds1987flocks}
C.~W. Reynolds, ``Flocks, herds and schools: A distributed behavioral model,''
  in \emph{Proceedings of the 14th annual conference on Computer graphics and
  interactive techniques}, 1987, pp. 25--34.

\bibitem{gong2020partial}
Q.~Gong, W.~Kang, C.~Walton, I.~Kaminer, and H.~Park, ``Partial observability
  analysis of an adversarial swarm model,'' \emph{Journal of Guidance, Control,
  and Dynamics}, vol.~43, no.~2, pp. 250--261, 2020.

\bibitem{walton2021defense}
C.~Walton, I.~Kaminer, Q.~Gong, A.~Clark, T.~Tsatsanifos \emph{et~al.},
  ``Defense against adversarial swarms with parameter uncertainty,''
  \emph{arXiv preprint arXiv:2108.04205}, 2021.

\bibitem{tsatsanifos2021modeling}
T.~Tsatsanifos, A.~H. Clark, C.~Walton, I.~Kaminer, and Q.~Gong, ``Modeling and
  control of large-scale adversarial swarm engagements,'' \emph{arXiv preprint
  arXiv:2108.02311}, 2021.

\bibitem{lien2004shepherding}
J.-M. Lien, O.~B. Bayazit, R.~T. Sowell, S.~Rodriguez, and N.~M. Amato,
  ``Shepherding behaviors,'' in \emph{IEEE International Conference on Robotics
  and Automation, 2004. Proceedings. ICRA'04. 2004}, vol.~4.\hskip 1em plus
  0.5em minus 0.4em\relax IEEE, 2004, pp. 4159--4164.

\bibitem{pierson2017controlling}
A.~Pierson and M.~Schwager, ``Controlling noncooperative herds with robotic
  herders,'' \emph{IEEE Transactions on Robotics}, vol.~34, no.~2, pp.
  517--525, 2017.

\bibitem{vaughan1998robotA}
R.~Vaughan, N.~Sumpter, J.~Henderson, A.~Frost, and S.~Cameron, ``Robot control
  of animal flocks,'' in \emph{Proceedings of the 1998 IEEE International
  Symposium on Intelligent Control (ISIC) held jointly with IEEE International
  Symposium on Computational Intelligence in Robotics and Automation (CIRA)
  Intell}.\hskip 1em plus 0.5em minus 0.4em\relax IEEE, 1998, pp. 277--282.

\bibitem{vaughan2000experiments}
------, ``Experiments in automatic flock control,'' \emph{Robotics and
  autonomous systems}, vol.~31, no. 1-2, pp. 109--117, 2000.

\bibitem{pierson2015bio}
A.~Pierson and M.~Schwager, ``Bio-inspired non-cooperative multi-robot
  herding,'' in \emph{2015 IEEE International Conference on Robotics and
  Automation (ICRA)}.\hskip 1em plus 0.5em minus 0.4em\relax IEEE, 2015, pp.
  1843--1849.

\bibitem{licitra2017singleA}
R.~A. Licitra, Z.~I. Bell, E.~A. Doucette, and W.~E. Dixon, ``Single agent
  indirect herding of multiple targets: A switched adaptive control approach,''
  \emph{IEEE Control Systems Letters}, vol.~2, no.~1, pp. 127--132, 2017.

\bibitem{licitra2017singleB}
R.~A. Licitra, Z.~D. Hutcheson, E.~A. Doucette, and W.~E. Dixon, ``Single agent
  herding of n-agents: A switched systems approach,'' \emph{IFAC-PapersOnLine},
  vol.~50, no.~1, pp. 14\,374--14\,379, 2017.

\bibitem{sebastian2021multi}
E.~Sebasti{\'a}n and E.~Montijano, ``Multi-robot implicit control of herds,''
  in \emph{2021 IEEE International Conference on Robotics and Automation
  (ICRA)}.\hskip 1em plus 0.5em minus 0.4em\relax IEEE, 2021, pp. 1601--1607.

\bibitem{bacon2012swarm}
M.~Bacon and N.~Olgac, ``Swarm herding using a region holding sliding mode
  controller,'' \emph{Journal of Vibration and Control}, vol.~18, no.~7, pp.
  1056--1066, 2012.

\bibitem{grover2020parameter}
J.~S. Grover, C.~Liu, and K.~Sycara, ``Parameter identification for multirobot
  systems using optimization based controllers (extended version),''
  \emph{arXiv preprint arXiv:2009.13817}, 2020.

\bibitem{grover2020feasible}
J.~Grover, C.~Liu, and K.~Sycara, ``Feasible region-based identification using
  duality (extended version),'' \emph{arXiv preprint arXiv:2011.04904}, 2020.

\bibitem{ames2019control}
A.~D. Ames, S.~Coogan, M.~Egerstedt, G.~Notomista, K.~Sreenath, and P.~Tabuada,
  ``Control barrier functions: Theory and applications,'' in \emph{2019 18th
  European Control Conference (ECC)}.\hskip 1em plus 0.5em minus 0.4em\relax
  IEEE, 2019, pp. 3420--3431.

\bibitem{wang2017safety}
L.~Wang, A.~D. Ames, and M.~Egerstedt, ``Safety barrier certificates for
  collisions-free multirobot systems,'' \emph{IEEE Transactions on Robotics},
  vol.~33, no.~3, pp. 661--674, 2017.

\end{thebibliography}

% \begin{thebibliography}{99}

% \bibitem{c1}
% J.G.F. Francis, The QR Transformation I, {\it Comput. J.}, vol. 4, 1961, pp 265-271.

% \bibitem{c2}
% H. Kwakernaak and R. Sivan, {\it Modern Signals and Systems}, Prentice Hall, Englewood Cliffs, NJ; 1991.

% \bibitem{c3}
% D. Boley and R. Maier, "A Parallel QR Algorithm for the Non-Symmetric Eigenvalue Algorithm", {\it in Third SIAM Conference on Applied Linear Algebra}, Madison, WI, 1988, pp. A20.

% \end{thebibliography}

\end{document}